%% file: arxiv_paper.tex
\documentclass[letterpaper]{article}
\usepackage[margin=1.1in]{geometry}

\usepackage[utf8]{inputenc}

\usepackage{cite}

\usepackage{nameref,hyperref}

\usepackage[aboveskip=1pt,labelfont=bf,labelsep=period,singlelinecheck=off]{caption}

\usepackage{lastpage,fancyhdr,graphicx}
\usepackage{epstopdf}
\usepackage{color}

\definecolor{Gray}{gray}{.25}

\usepackage{graphicx}

\usepackage{sidecap}

\usepackage{wrapfig}
\usepackage[pscoord]{eso-pic}
\usepackage[fulladjust]{marginnote}
\reversemarginpar

\def\MM{\mathcal{M}}

\def\SSS{\mathcal{S}}

\def\XX{\mathcal{X}}

\def\ind{\mathbbm{1}}
\def\dpptester{\texttt{DPP-\textsc{Tester}}}
\def\dpptesterr{\texttt{DPP-\textsc{Tester2}}}

\usepackage[dvipsnames,table,xcdraw]{xcolor}

\usepackage{xspace}
\usepackage[numbers]{natbib}
\usepackage{palatino}
\usepackage{enumerate}
\usepackage[shortlabels]{enumitem}
\usepackage{amsthm,amssymb}
\usepackage{bbm}
\usepackage[title]{appendix}
\usepackage{mathtools}
\usepackage{xcolor}
\usepackage{algorithm}
\usepackage[noend]{algpseudocode}
\usepackage{bbm}
\usepackage{amsmath}

\newcommand{\J}{\mathcal{J}}
\newcommand{\dpp}[2]{\mathbf{Pr}_{#1}[#2]}

\title{Testing Determinantal Point Processes}

\author{%
  Khashayar Gatmiry
  \\MIT\\
  \texttt{gatmiry@mit.edu} 
   \and
   Maryam Aliakbarpour \\
   MIT\\
   \texttt{maryama@mit.edu}\\
   \and
   Stefanie Jegelka\\
   MIT\\
   \texttt{stefje@mit.edu}
}

\date{}

\newtheorem{theorem}{Theorem}

\newtheorem{lemma}{Lemma}

\newtheorem{definition}{Definition}

\renewcommand\Pr[2][]{
\ensuremath{\mathrm{\mathbf{Pr}}_{#1} 
\pmb{\left[\vphantom{#2}\right.}
{#2}
\pmb{\left.\vphantom{#2}\right]}
}}
\newcommand\numberthis{\addtocounter{equation}{1}\tag{\theequation}}
\makeatletter
\newcommand*{\inlineequation}[2][]{%
  \begingroup

\newcommand\E[2][]{
\ensuremath{\mathrm{\mathbf{E}}_{#1} 
\pmb{\left[\vphantom{#2}\right.}
{#2}
\pmb{\left.\vphantom{#2}\right]}
}}

    \refstepcounter{equation}%
    \ifx\\#1\\%
    \else
      \label{#1}%
    \fi
    \relpenalty=10000 %
    \binoppenalty=10000 %
    \ensuremath{%
      #2%
    }%
    ~\@eqnnum
  \endgroup
}
\makeatother

\def\accept{{\fontfamily{cmss}\selectfont accept}\xspace}
\def\reject{{\fontfamily{cmss}\selectfont reject}\xspace}

\newcommand{\sj}[1]{\textcolor{blue}{SJ: #1}}
\newcommand{\kh}[1]{\textcolor{red}{KH: #1}}

 \newtheorem{innercustomthm}{Theorem}
\newenvironment{customthm}[1]
  {\renewcommand\theinnercustomthm{#1}\innercustomthm}
  {\endinnercustomthm}
 \newtheorem{innercustomdef}{Definition}
\newenvironment{customdef}[1]
  {\renewcommand\theinnercustomthm{#1}\innercustomdef}
  {\endinnercustomthm}
\begin{document}

\maketitle

\begin{abstract}
  Determinantal point processes (DPPs) are popular probabilistic models of diversity. In this paper, we investigate DPPs from a new perspective: property testing of distributions. Given sample access to an unknown distribution $q$ over the subsets of a ground set, we aim to distinguish whether $q$ is a DPP distribution, or $\epsilon$-far from all DPP distributions in $\ell_1$-distance. In this work, we propose the first algorithm for testing DPPs. Furthermore, we establish a matching lower bound on the sample complexity of DPP testing.
  This lower bound also extends to showing a new hardness result for the problem of testing the more general class of log-submodular distributions. 
\end{abstract}

\input{intro}

\input{learner.tex}

\input{DPP_testing_part.tex}

\input{LB.tex}

\input{tradeoffs.tex}

\subsection*{Acknowledgments}
This research was supported in part by NSF CAREER award 1553284 and NSF BIGDATA award IIS-1741341, NSF award numbers CCF-1733808, NSF BIGDATA award IIS-1741137, and MIT-IBM Watson AI Lab and Research Collaboration Agreement No. W1771646.
The authors would like to thank Ankur Moitra and Behrooz Tahmasebi for helpful discussions.

\bibliographystyle{plainnat}
\bibliography{arxiv_paper}

\appendix
\input{appendixA.tex}
\input{appendixB.tex}
\input{appendixI.tex}
\input{appendixH.tex}
\input{appendixJ.tex}

\input{appendixD.tex}
\input{appendixC.tex}
\input{appendixF.tex}
\input{appendixE.tex}

\end{document}

%% file: intro.tex
\section{Introduction}
Determinantal point processes (DPPs) are a rich class of discrete probability distributions that were first studied in the context of quantum physics~\citep{macchi1975coincidence} and random matrix theory~\cite{dyson1962statistical}. Initiated by the seminal work of \citet{kulesza2012determinantal}, DPPs have gained a lot of attention in machine learning, due to their ability to naturally capture notions of diversity and repulsion. Moreover, they are easy to define via a similarity (kernel) matrix, and, as opposed to many other probabilistic models, offer tractable exact algorithms for marginalization, conditioning and sampling \cite{anari16,hough06,kulesza2012determinantal,li16sr}. Therefore, DPPs have been explored in a wide range of applications, including video summarization \cite{gong2014diverse,NIPS2014_5413}, image search~\cite{kulesza2011k, affandi2014learning},  document and timeline summarization~\cite{lin2012learning}, recommendation \cite{wilhelm18youtube}, feature selection in bioinformatics~\cite{batmanghelich2014diversifying}, modeling neurons~\cite{snoek2013determinantal}, and matrix approximation \cite{derezinski20survey,deshpande06,li16kernel}.

A \emph{Determinantal Point Process} is a distribution over the subsets of a ground set $[n] = \{1,2, \ldots n\}$, and parameterized by a 
\emph{marginal kernel} matrix $K \in \mathbb{R}^{n \times n}$ with eigenvalues in $[0,1]$, whose $(i,j)$th entry expresses the similarity of items $i$ and $j$. Specifically, 
the marginal probability that a set $A \subseteq [n]$ is observed in a random $\J \sim \dpp{K}{.}$ is $\mathbb P(A \subseteq \mathcal J) = \det(K_A)$, where $K_A$ is the principal submatrix of $K$ indexed by $A$. Hence, similar items are less likely to co-occur in $\J$.

Despite the wide theoretical and applied literature on DPPs, one question has not yet been addressed: \emph{Given a sample of subsets, can we test whether it was generated by a DPP?} This question arises, for example, when trying to decide whether a DPP may be a suitable mathematical model for a dataset at hand. To answer this question, we study DPPs from the perspective of \emph{property testing}. Property testing aims to decide whether a given distribution has a property of interest, by observing as few samples as possible. In the past two decades, property testing has received a lot of attention, and questions such as testing uniformity, independence, identity to a known or an unknown given distribution, and monotonicity have been studied in this framework \cite{Clement_survey, Rub12}.

More precisely, we ask \emph{How many samples from an unknown distribution are required to distinguish, with high probability, whether it is a DPP or $\epsilon$-far from the class of DPPs in $\ell_1$-distance?}
Given the rich mathematical structure of DPPs, one may hope for a tester that is exceptionally efficient. Yet, we show that testing is still not easy, and establish a lower bound of $ \Omega(\sqrt{N}/\epsilon^2)$ for the sample size of any valid tester, where $N = 2^n$ is the size of the domain. 
In fact, this lower bound applies to the broader class of \emph{log-submodular} measures, and may hence be of wider interest given the popularity of submodular set functions in machine learning. Even more generally, the lower bound holds for testing \emph{any} subset of log-submodular distributions that include the uniform measure, and reveals a large gap between sample access and query access. 

We note that the $\sqrt{N}$ dependence on the domain size is not uncommon in distribution testing, since it is required even for testing simple structures such as uniform distributions~\cite{Paninski:08}. 
However, achieving the optimal sample complexity is nontrivial. We provide the first algorithm for testing DPPs; it uses $\Tilde{O}(\sqrt{N}/\epsilon^2)$ samples.
This algorithm achieves the lower bound and hence settles the complexity of testing DPPs. Moreover, we show how prior knowledge on bounds of the spectrum of $K$ or its entries $K_{ij}$ can improve logarithmic factors in the sample complexity.
Our approach relies on {\em testing via learning}. As a byproduct, our algorithm is the first to provably learn a DPP in $\ell_1$-distance, while previous learning approaches only considered parameter recovery in $K$ \cite{urschel2017learning,pmlr-v65-brunel17a}, which does not imply recovery in $\ell_1$-distance. 

In short, we make the following contributions:
\vspace{-5pt}
\begin{itemize}
    \item We show a lower bound of $\Omega(\sqrt{N}/\epsilon^2)$ for the sample complexity of testing any subset of the class of \emph{log-submodular} measures which includes the uniform measure, implying the same lower bound for testing DPP distributions and SR measures. 
    
    \item We provide the first tester for the family of DPP distributions using $\Tilde{O}(\sqrt{N}/\epsilon^2)$ samples. The sample complexity is optimal with respect to $\epsilon$ and the domain size $N$, up to logarithmic factors, and does not depend on other parameters.
    Additional assumptions on $K$ can improve the algorithm's complexity.
    
    \item As a byproduct of our algorithm, we give the first algorithm to learn DPP distributions in both $\ell_1$ and $\chi^2$ distances.
\end{itemize}

\section{Related work}\label{sec:related}
\textbf{Distribution testing.} 
{\em Hypothesis testing} is a classical tool in statistics for inference about the data and model~\cite{NeymanP1933, lehmann2005testing}. About two decades ago, the framework of {\em distribution testing} was introduced, to view such 
statistical problems from a computational perspective~\cite{GR00,BFRSW}. This framework is a branch of {\em property testing}~\cite{RubinfeldS96}, and focuses mostly on discrete distributions. Property testing analyzes the non-asymptotic performance of algorithms, i.e., for finite sample sizes.
By now, distribution testing has been studied extensively for properties such as uniformity~\cite{Paninski:08}, identity to a known~\cite{BatuFFKRW, ADK15, diakonikolasGPP18} or unknown distribution~\cite{ChanDVV14, DiakonikolasK:2016}, independence~\cite{BatuFFKRW}, monotonicity~\cite{BatuKR04, aliakbarpourGPRY18}, k-modality~\cite{DaskalakisDS:2014}, entropy estimation~\cite{entropy,  WuY16Entropy}, and support size estimation~\cite{supportSizeDana, VV11, wu2019chebyshev}. The surveys \cite{Clement_survey, Rub12} provide further details. 

\textbf{Testing submodularity and real stability.}  Property testing also includes testing properties of functions. As opposed to distribution testing, where observed samples are given, testing functions allows an active query model: given query access to a function $f: \mathcal X \rightarrow \mathcal Y$, the algorithm picks points $x \in \XX$ and obtains values $f(x)$. The goal is again to determine, with as few queries as possible, whether $f$ has a given property or is $\epsilon$-far from it. Closest to our work in this different model is the question of testing submodularity, in Hamming distance and $\ell_p$-distance~\cite{chakrabarty2012testing,seshadhri2014submodularity, feldman2016optimal, blais2016testing}, since any DPP-distribution is log-submodular. In particular, \citet{blais2016testing} show that testing submodularity with respect to any $\ell_p$ norm is feasible with a constant number of queries, independent of the function's domain size. The vast difference between this result and our lower bound for log-submodular distributions lies in the query model -- given samples versus active queries -- and demonstrates the large impact of the query model.
DPPs also belong to the family of \emph{strongly Rayleigh} measures \cite{borcea09}, whose generating functions are real stable polynomials. \citet{raghavendra17} develop an algorithm for testing real stability of bivariate polynomials, which, if nonnegative, correspond to distributions over two items.

\textbf{Learning DPPs.} The problem of learning DPPs has been of great interest in machine learning. Unlike testing, in learning one commonly assumes that the underlying distribution is indeed a DPP, and aims to estimate the marginal kernel $K$. It is well-known that maximum likelihood estimation for DPPs is a highly non-concave optimization problem, conjectured to be NP-hard~\cite{pmlr-v65-brunel17a, KuleszaThesis}. To circumvent this difficulty, previous work imposes additional assumptions, e.g., a parametric family for $K$~\cite{kulesza2011k,10.5555/3020548.3020597,  affandi2014learning, kulesza2012determinantal, bardenet2015inference,lavancier2015determinantal}, or low-rank structure \cite{gartrell2016bayesian, gartrell2017low, pmlr-v84-dupuy18a}. A variety of optimization and sampling techniques have been used, e.g., variational methods~\cite{djolonga2014map, gillenwater2014expectation, bardenet2015inference}, MCMC~\cite{affandi2014learning}, first order methods~\cite{kulesza2012determinantal}, and fixed point algorithms~\cite{mariet2015fixed}. 
\citet{pmlr-v65-brunel17a} analyze the asymptotic convergence rate of the Maximum likelihood estimator. To avoid likelihood maximization, \citet{urschel2017learning} propose an algorithm based on the method of moments, with statistical guarantees. Its complexity is determined by the \emph{cycle sparsity} property of the DPP. We further discuss the implications of their result in our context in Section~\ref{sec:mainresults}. Using similar techniques,~\citet{brunel2018learning} considers learning the class of signed DPPs, i.e., DPPs that allow skew-symmetry, $K_{i,j} = \pm K_{j,i}$.

\section{Notation and definitions} \label{sec:notations}
Throughout the paper, we consider discrete probability distributions over subsets of a \emph{ground set} $[n] = \{1, 2, \ldots, n\}$, i.e., over the power set $2^{[n]}$ of size $N\coloneqq2^n$. We refer to such distributions via their probability mass function $p:2^{[n]} \rightarrow \mathbb{R}^{\geq 0}$ satisfying $\sum_{S \subseteq [n]} p(S) = 1$. For two distributions $p$ and $q$, we use $\ell_1(q,p) = \frac{1}{2}\sum_{S \subseteq [n]} |q(S) - p(S)|$ to indicate their $\ell_1$ (total variation) distance, and $\chi^2(q,p) = \sum_{S \subseteq [n]} \frac{(q(S)-p(S))^2}{p(S)}$ to indicate their $\chi^2$-distance.

\paragraph{Determinantal Point Processes (DPPs).} A DPP is a discrete probability distribution parameterized by a positive semidefinite kernel matrix $K \in \mathbb{R}^{n \times n}$, with eigenvalues in $[0,1]$. More precisely, the marginal probability for any set $S \subseteq [n]$ to occur in a sampled set $\J$ is given by the principal submatrix indexed by rows and columns in $S$:  $\Pr[\mathcal J \sim K]{S \subseteq \mathcal J} = \det(K_{S}).$
We refer to the probability mass function of the DPP by $\Pr[K]{J} = \Pr[\mathcal J \sim K]{\mathcal J = J}$. A simple application of the inclusion-exclusion principle reveals an expression in terms of the complement $\bar{J}$ of $J$:

\begin{equation}
\Pr[K]{J} = |\det(K - I_{\bar{J}})|. \label{eq:atomprobabilities}
\end{equation}

\textbf{Distribution testing.} We mathematically define a \emph{property} $\mathcal P$ to be a set of distributions. A distribution $q$ has the property $\mathcal P$ if $q \in \mathcal P$.
We say two distributions $p$ and $q$ are {\em $\epsilon$-far} from ({\em $\epsilon$-close} to) each other, if and only their $\ell_1$-distance is at least (at most) $\epsilon$. Also, $q$ is $\epsilon$-far from $\mathcal P$ if and only if it is $\epsilon$-far from any distribution in $\mathcal P$. We define the \emph{$\epsilon$-far set} of $\mathcal P$ to be the set of all distributions that are $\epsilon$-far from $\mathcal P$. We say an algorithm is an {\em $(\epsilon, \delta)$-tester} for property $\mathcal P$ if, upon receiving samples from an unknown distribution $q$, the following is true with probability at least $1-\delta$:
\vspace{-5pt}
\begin{itemize}\setlength{\itemsep}{0pt}
    \item If $q$ has the property $\mathcal P$, then the algorithm outputs \accept.
    \item If $q$ is $\epsilon$-far from $\mathcal P$, then the algorithm outputs \reject.
\end{itemize}
\vspace{-5pt}
We refer to  $\epsilon$ and $\delta$ as {\em proximity parameter} and {\em confidence parameter}, respectively. 
Note that if we have an $(\epsilon, \delta)$-tester for a property with a confidence parameter $\delta < 0.5$, then we can achieve an $(\epsilon, \delta')$-tester for an arbitrarily small $\delta'$ by multiplying the sample size by an extra factor of $\Theta(\log (\delta/{\delta'}))$. This {\em amplification} technique~\cite{dubhashi1998concentration} is a direct implication of the Chernoff bound when we run the initial tester $\Theta(\log (\delta/{\delta'}))$ times and take the majority output as the answer.

\section{Main results} \label{sec:mainresults}
We begin by summarizing our main results, and explain more proof details in Sections~\ref{sec:learning} and \ref{section:lowerbound}.

\textbf{Upper bound.} Our first result is the first upper bound on the sample complexity of testing DPPs.

\begin{theorem}[Upper Bound] \label{thm:upperboundtheorem}
Given samples from an unknown distribution $q$ over $2^{[n]}$, there exists an $(\epsilon, 0.99)$-tester for determining whether $q$ is a DPP or it is $\epsilon$-far from all DPP distributions. The tester uses
\begin{align}
O(C_{N,\epsilon} \sqrt N/\epsilon^{2} ) \label{eq:upperboundcomplexity}    
\end{align}
samples with logarithmic factors 
$C_{N,\epsilon} =  \log^2(N) (\log(N) + \log(1/\epsilon))$.
\end{theorem}
Importantly, the sample complexity of our upper bound grows as $\tilde{O}(\sqrt{N}/\epsilon^2)$, which is optimal up to a logarithmic factor (Theorem~\ref{thm:lowerboundtheorem}).
With additional assumptions on the spectrum and entries of $K$, expressed as \emph{$(\alpha, \zeta)$-normal} DPPs, we obtain a refined analysis.
\begin{definition} \label{def:normalitydefinition}
  For $\zeta \in [0, 0.5]$ and $\alpha \in [0,1]$, a DPP with marginal kernel $K$ is \emph{$(\alpha, \zeta)$-normal} if: 
  \vspace{-5pt}
\begin{enumerate}\setlength{\itemsep}{0pt}
    \item The eigenvalues of $K$ are in the range $[\zeta, 1 - \zeta]$; and
    \item For $i,j \in [n]: K_{i,j} \neq 0 \Rightarrow |K_{i,j}| \geq \alpha$.
\end{enumerate}
\end{definition}
The notion of $\alpha$-normal DPPs was also used in~\cite{urschel2017learning}.
Since $K$ has eigenvalues in $[0,1]$, its entries $K_{i,j}$ are at most one. Hence, we always assume $0 \leq \zeta \leq 0.5$ and $0 \leq \alpha \leq 1$. 

\begin{lemma}\label{prop:refined}
  For $(\alpha, \zeta)$-normal DPPs, with knowledge of $\alpha$ and $\zeta$, the factor in Theorem~\ref{thm:upperboundtheorem} becomes $C'_{N,\epsilon, \zeta, \alpha} =  \log^2(N) (1 + \log(1/\zeta) + \min\{ \log (1/\epsilon) , \log (1/\alpha)\})$.
\end{lemma}

Even more, if at least one of $\epsilon$ or $\alpha$ is not too small, i.e., if $\epsilon = \tilde{\Omega}(\zeta^{-2} N^{-1/4})$ or $\alpha = \tilde{\Omega}(\zeta^{-1} N^{-1/4})$ hold, then $C'_{N,\epsilon, \zeta, \alpha}$ reduces to $\log^2(N)$.
With a minor change in the algorithm, the bound in Lemma~\ref{prop:refined} also holds for the problem of testing whether $q$ is an $(\alpha, \zeta)$-normal DPP, or $\epsilon$-far only from just the class of $(\alpha, \zeta)$-normal DPPs, instead of all DPPs (Appendix~\ref{app:appendixd}).

Our approach tests DPP distributions via {\em learning}: At a high-level, we learn a DPP model from the data as if the data is generated from a DPP distribution. Then, we use a new batch of data and test whether the DPP we learnt seems to have generated the new batch of the data.
More accurately, given samples from $q$, we pretend $q$ is a DPP with kernel $K^*$, and use a proper learning algorithm to estimate a kernel matrix $\hat{K}$. 
But, the currently best learning algorithm \cite{urschel2017learning} has a   lower bound on the sample complexity of learning $K^*$ accurately which, 
in the worst case, may lead to a sub-optimal sample complexity for testing. 

To reduce the sample complexity of learning, we do not work with a single accurate estimate $\hat K$, but construct a set $\mathcal M$ of candidate DPPs as potential estimates for $q$. We show that, with only $\Theta(\sqrt N/\epsilon^2)$ samples, we can obtain a set $\mathcal M$ such that, with high probability, we can determine if $q$ is a DPP by testing if $q$ is close to any DPP in $\MM$. We prove that $\Theta(\log(|\MM|) \sqrt{N} /\epsilon^2)$ samples suffice for
this algorithm to succeed with high probability.

\textbf{Lower Bound.} Our second main result is an information-theoretic lower bound, which shows that the sample complexity of our tester in Theorem~\ref{thm:upperboundtheorem} is optimal up to logarithmic factors.
\begin{theorem}[Lower Bound] \label{thm:lowerboundtheorem}
  Given $\epsilon \leq 0.0005$, $n \geq 22$, and $\alpha \in [0 , 0.5]$, any $(\epsilon, 0.99)$-tester needs at least $\Omega(\sqrt{N}/\epsilon^{2})$ samples to distinguish if $q$ is a DPP or it is $\epsilon$-far from the class of DPPs.

The same bound holds for distinguishing if $q$ is an $(\alpha, \zeta)$-normal DPP or it is $\epsilon$-far from the class of DPPs (or $\epsilon$-far from the class of $(\alpha, \zeta)$-normal DPPs).
\end{theorem}

In fact, we prove a more general result (Theorem~\ref{thm:logsubmodular}): testing whether $q$ is in any subclass $\Upsilon$ of the family of log-submodular distributions that includes the uniform distribution requires $\Omega(\sqrt{N}/\epsilon^2)$ samples. DPPs are such a subclass~\cite{kulesza2012determinantal}. A distribution $f$ over $2^{[n]}$ is \emph{log-submodular} if for every $S \subset S' \subseteq [n]$ and $i \notin S'$, it holds that 
$\log(f(S' \cup \{i\})) - \log(f(S')) \leq 
\log(f(S \cup \{i\})) - \log(f(S))$.

Given the interest in log-submodular distributions \cite{djolonga2014map,tschiatschek16,djolonga18,gotovos15sampling,gotovos18discrete}, this result may be of wider interest. Moreover, our lower bound applies to another important subclass $\Upsilon$, \emph{strongly Rayleigh measures} \cite{borcea09}, which underlie recent progress in algorithms and mathematics \cite{gharan09,frieze14,spielmanSri11,anari15ks}, and sampling in machine learning \cite{anari16,li17dual,li16sr}.

Our lower bound stands in stark contrast to the \emph{constant} sample complexity of testing whether a given function is submodular \cite{blais2016testing}, implying a wide complexity gap between access to given samples and access to an evaluation oracle (see Section~\ref{sec:related}).
We prove our lower bounds by a reduction from a randomized instance of uniformity testing.

%% file: learner.tex
\def\Kt{\Tilde{K}}

\section{An Algorithm for Testing DPPs} \label{sec:learning}
We will first construct an algorithm for testing the smaller class of $(\alpha,\zeta)$-normal DPPs, and then show how to extend this result to all DPPs via a coupling argument.

Our testing algorithm relies on learning: given samples from $q$, we estimate a kernel $\hat K$ from the data, and then test whether the estimated DPP has generated the observed samples. The magnitude of any entry $\hat{K}_{i,j}$ can be estimated from the marginals for $S = \{i,j\}$ and $i,j$, since $\dpp{K}{S} = K_{i,i}K_{j,j} - K_{i,j}^2 = \dpp{K}{i}\dpp{K}{j} - K_{i,j}^2$. But, determinig the signs is more challenging. \citet{urschel2017learning} estimate signs via higher order moments that are harder to estimate, but it is not clear whether the resulting estimate $\hat{K}$ yields a sufficiently accurate estimate of the distribution to obtain an optimal sample complexity for testing.
Hence, instead, we construct a set $\MM$ of candidate DPPs such that, with high probability, there is a $\tilde p \in \mathcal M$ that is close to $q$ if and only if $q$ is a DPP. Our tester, Algorithm~\ref{alg:testing_DPP}, tests closeness to $\MM$ by individually testing closeness of each candidate in $\MM$.

\begin{algorithm}[t]
\caption{\dpptester \label{alg:testing_DPP}}
\begin{algorithmic}[1]
\Procedure{DPP-Tester}{$\epsilon$, $\delta$, sample access to $q$}
    \State{$\MM \gets $ construct the set of DPP distributions as described in Theorem~\ref{thm:learning}.}
    \For {each $p$ in $\MM$}
        \State{Use robust $\chi^2-\ell_1$ testing to check if $\chi^2(q,p) \leq \epsilon^2/500$, or $\ell_1(q,p) \geq \epsilon$}.
        \If{the tester outputs \accept}
            \State{\textbf{Return} \accept.}
        \EndIf
    \EndFor
 	\State{\textbf{Return} \reject}
\EndProcedure
\end{algorithmic}
\end{algorithm}

\textbf{Constructing $\MM$.} The DPPs in $\MM$ arise from variations of an estimate for $K^*$, obtained with $\Theta(\sqrt N/\epsilon^2)$ samples. Via the above stragegy, we first estimate the magnitude $|K^*_{ij}|$ of each matrix entry, and then pick candidate entries from the confidence intervals around $\pm|\widehat{K}_{ij}|$, such that at least one is close to the true $K^*_{i,j}$. The candidate matrices $K$ are obtained by all possible combinations of candidate entries. Since these are not necessarily valid marginal kernels, we project them onto the positive semidefinite matrices with eigenvalues in $[0,1]$. Then, $\MM$ is the set of all DPPs parameterized by these projected candidate matrices $\Pi(K)$.

If $q$ is a DPP with kernel $K^*$, then, by construction, our candidates contain  a $\tilde{K}$ close to $K^*$. The projection of $\tilde K$ remains close to $K^*$ in Frobenius distance. We show that this closeness of the matrices implies closeness of the corresponding distributions $q$ and $\tilde p = \Pr[\Pi(\tilde K)]{.}$ in $\ell_1$-distance: $\ell_1(q,\tilde p) = O(\epsilon)$. Conversely, if $q$ is $\epsilon$-far from being a DPP, then it is, by definition, $\epsilon$-far from $\MM$, which is a subset of all DPPs.

\textbf{Testing $\MM$.} To test whether $q$ is close to $\MM$, a first idea is to do robust $\ell_1$ identity testing, i.e., for every $p \in \mathcal M$, test whether $\ell_1(q,p) \geq \epsilon$ or $\ell_1(q,p) = O(\epsilon)$. But, $\MM$ can contain the uniform distribution, and it is known that robust $\ell_1$ uniformity testing needs $\Omega(N/\log N)$ samples~\cite{VV11}, as opposed to the optimal dependence $\sqrt N$.

Hence, instead, we use a combination of $\chi^2$ and $\ell_1$ distances for testing, and test $\chi^2(q,p) = O(\epsilon^2)$ versus $\ell_1(q,p) \geq \epsilon$ \citep{ADK15}. To apply this robust $\chi^2$-$\ell_1$ identity testing, we must prove that, with high probability, there is a $\tilde{p}$ in $\MM$ with $\chi^2(q,\tilde p) = O(\epsilon^2)$ if and only if $q$ is a DPP. Theorem~\ref{thm:learning}, proved in Appendix~\ref{app:appendixa}, asserts this result if $q$ is an $(\alpha, \zeta)$-normal DPP. This is stronger than its $\ell_1$ correspondent, since $\frac{1}{2}\ell_1^2(q,\tilde p) \leq \chi^2(q,\tilde p)$. 

To prove Theorem~\ref{thm:learning}, we need to control the distance between the atom probabilities of $q$ and $\tilde p$. We analyze these atom probabilities, which are given by determinants, via a lower bound on the smallest singular values $\sigma_n$ of the family of matrices $\{ K - I_{\bar J}: \ J \subseteq [n] \}$.
\begin{lemma} \label{smallest-singular-value-lemma}
If the kernel matrix $K$ has all eigenvalues in $[\zeta, 1 - \zeta]$, then, for every $J \subseteq [n]$:
$$\sigma_n(K - I_{\bar J}) \geq \zeta (1 - \zeta)/\sqrt 2.$$
\end{lemma}
Lemma~\ref{smallest-singular-value-lemma} is proved in Appendix~\ref{app:appendixb}. In Theorem~\ref{thm:learning}, we observe $m = \lceil (\ln(1/\delta) +  1)\sqrt{N}/\epsilon^2\rceil$ samples from $q$, and use the parameter $\varsigma \coloneqq \lceil 200 n^2\zeta^{-1}  \min\{2\xi/\alpha, \sqrt{\xi/\epsilon} \} \rceil$, with $\xi \coloneqq N^{-\frac{1}{4}} \sqrt{\log(n) + 1}$. 
%
\begin{theorem} \label{thm:learning}
Let $q$ be an $(\alpha, \zeta)$-normal DPP distribution with marginal kernel $K^*$. Given the parameters defined above, suppose we have $m$ samples from $q$. 
Then, one can generate a  set $\mathcal{M}$ of DPP distributions with cardinality $|\mathcal{M}| = (2\varsigma+1)^{n^2}$, such that, with probability at least $1 - \delta$, there is a distribution $\tilde p \in \mathcal{M}$ with $\chi^2(q,\tilde p) \leq \epsilon^2/500$. 
\end{theorem}

%% file: DPP_testing_part.tex
\subsection{Correctness of the Testing Algorithm}\label{subsec:testing}
For simplicity of exposition, in Algorithm~\ref{alg:testing_DPP} we set the confidence parameter $\delta=0.01$.
We first prove the adaptive sample complexity in Lemma~\ref{prop:refined} for testing $(\alpha, \zeta)$-normal DPPs. 
Therefore, \dpptester\ aims to output \accept\ if $q$ is a $(\alpha,\zeta)$-normal DPP, and \reject\ if $q$ is $\epsilon$-far from all DPPs, in both cases with probability at least $0.99$.

To finish the proof for the adaptive sample complexity,
we need to argue that our \dpptester\ succeeds with high probability, i.e., that with high probability all of the identity tests, with each $p \in \MM$, succeed.
The algorithm uses robust $\chi^2$-$\ell_1$ identity testing \citep{ADK15}, to test $\chi^2(q,p) \leq \epsilon^2/500$ versus $\ell_1(q,p) \geq \epsilon$.
For each $p \in \MM$, the $\chi^2$-$\ell_1$ tester computes the statistic
\begin{equation} \label{eq:statistic}
    Z^{(m)} = \sum_{J \subseteq [n]: \, p(J) \geq \epsilon/50N} \frac{(N(J) - m p(J))^2 - N(J)}{m p(J)},
\end{equation}
where $m$ is the number of observed samples and $N(J)$ is the number of samples that are equal to set $J$, and compares $Z^{(m)}$ with the threshold $C = m\epsilon^2/10$.

\citet{ADK15} show that, given $\Theta(\sqrt N/\epsilon^2)$ samples, $Z^{(m)}$ concentrates around its mean, which is strictly below $C$ if $p$ satisfies $\chi^2(q,p) \leq \epsilon^2/500$, and strictly above $C$ if $\ell_1(q,p) \geq \epsilon$.
Let $\mathcal E_1$ be the event that all these robust tests, for every $p \in \mathcal M$, simultaneously answer correctly. To make sure that $\mathcal E_1$ happens with high probability, we use amplification (Section~\ref{sec:notations}): while we use the same set of samples to test against every $p \in \mathcal M$, we multiply the sample size by $\Theta(\log(|\MM|))$ to be confident that each test answers correctly with probability at least $1 - O(|\MM|^{-1})$. A union bound then implies that $\mathcal E_1$ happens with arbitrarily large constant probability.

Theorem~\ref{thm:learning} states that, indeed, with $\Theta(\sqrt N / \epsilon^2)$ samples, if $q$ is an $(\alpha, \zeta)$-normal DPP, then $\MM$ contains a distribution $\tilde p$ such that $\chi^2(q,\tilde p) \leq \epsilon^2 / 500$, with high probability. We call this event $\mathcal{E}_2$. \dpptester\ succeeds in the case $\mathcal{E}_1 \cap \mathcal{E}_2$: If $q$ is an $(\alpha, \zeta)$-normal DPP, then at least one $\chi^2$-$\ell_1$ test accepts $\tilde p$ and consequently the algorithm accepts $q$ as a DPP. Conversely, if $q$ is $\epsilon$-far from all DPPs, then $\ell_1(q,p) \geq \epsilon$ for every $p \in \MM$, so all the $\chi^2$-$\ell_1$ tests reject simultaneously and \dpptester \, rejects $q$ as well.
With a union bound on the events $\mathcal E_1^c$ and $\mathcal E_2^c$, it follows that  $\mathcal E_1 \cap \mathcal E_2$ happens with arbitrarily large constant probability too, independent of whether $q$ is a DPP or not. 

Adding the sample complexities for generating $\mathcal M$ and for the $\chi^2$-$\ell_1$ tests and observing $\log(|\MM|) = O(1 +  \log(1/\zeta) + \min\{ \log (1/\epsilon) , \log (1/\alpha)\} )$ completes the proof of Lemma~\ref{prop:refined}.  

\subsection{Extension to general DPPs} \label{subsec:extention}
Next, we prove the general sample complexity with factor $C_{N,\epsilon}$ in Theorem~\ref{thm:upperboundtheorem}. The key idea is that, if some eigenvalue of $K^*$ is very close to zero or one, we couple the process of sampling from $K^*$ with sampling from another kernel $\Pi_{z}(K^*)$ whose eigenvalues are bounded away from zero and one, i.e., parameterizing a $(0,z)$-normal DPP.
This coupling enables us to test $(0,z)$-normal DPPs instead, by tolerating an extra failure probability,
and transfer the above analysis for $(\alpha,\zeta)$-normal DPPs.
We state our coupling argument in the following Lemma, proved in Appendix~\ref{app:appendixh}.
\begin{lemma} \label{lemma:coupling}
For a value $z \in [0,1]$, we denote the projection of a marginal kernel $K$ onto the convex set $\{A \in S_n^+| \,\, z.I \leq A \leq (1-z)I \}$ by $\Pi_{z}(K)$, where $S_n^+$ is the set of positive semidefinite matrices. For $z = \delta/2mn$, consider the following stochastic processes:
\begin{enumerate}
    \item derive $m$ i.i.d samples $\{\mathcal J^{(t)}_K\}_{t=1}^m$ from $\Pr[K]{.}$.
    \item derive $m$ i.i.d samples $\{\mathcal J^{(t)}_{\Pi_z(K)}\}_{t=1}^m$ from $\Pr[\Pi_{z}(K)]{.}$.
\end{enumerate}
There exist a coupling between $(1)$ and 
$(2)$ such that
$$\Pr[\text{coupling}]{\{\mathcal J^{(t)}_{K}\}_{t=1}^m = \{\mathcal J^{(t)}_{\Pi_z(K)}\}_{t=1}^m} \geq 1-\delta.$$
\end{lemma}

We can use this coupling argument as follows. Suppose the constant $c_1$ is such that using \\ $c_1 C_{N,\epsilon,\alpha,\zeta} \sqrt{N} / \epsilon^2$ samples suffice for~\dpptester\  to output the correct answer for testing $(\alpha, \zeta)$-normal DPPs, with probability at least $0.995$. Such a constant exists as we just proved. Now, we show that with $m^* = c_2 C_{N,\epsilon} \sqrt N/\epsilon^2$ samples for large enough constant $c_2$, we obtain a tester for the set of all DPPs. To this end, we use the parameter setting of our algorithm for $(0,\bar z)$ normal DPPs, where $\bar z = 0.005/2m^*n$ is a function of $c_2$, $\epsilon$, and $N$. One can readily see that $c_2$ can be picked large enough, such that $m^* \geq c_1 C_{N,\epsilon,0,\bar z} \sqrt{N} / \epsilon^2$, with $c_2$ being just a function of $c_1$. This way, by the definition of $c_1$, 
the algorithm can test for $(0,\bar z)$-normal DPPs with success probability $0.995$. So, if $q$ is a $(0,\bar z)$-normal DPP, or  if it is $\epsilon$-far from all DPPs, then the algorithm outputs correctly with probability at least $0.995$. 

It remains to check what happens when $q$ is a DPP with kernel $K^*$, but not $(0,\bar{z})$-normal.
Indeed, \dpptester\ successfully decides this case too: due to our coupling, the product distributions $\mathrm{\mathbf{Pr}}^{(m^*)}_{K^*}[.]$ and $\mathrm{\mathbf{Pr}}^{(m^*)}_{\Pi_{\bar z}(K^*)}[.]$ over the space of data sets have $\ell_1$-distance at most $0.005$, so we have 
\\ $\mathrm{\mathbf{Pr}}^{(m^*)}_{K^*}\left[ \text{Acceptance Region}\right]
\geq 
\mathrm{\mathbf{Pr}}^{(m^*)}_{\Pi_{\bar z}(K^*)}\left[ \text{Acceptance Region}\right] - 0.005 \geq 0.995 - 0.005 = 0.99$, where the last inequality follows from the fact that $\Pi_{\bar z}(K^*)$ is an $(0,\bar z)$-normal DPP.
Hence, for such $c_2$, \dpptester\ succeeds with $c_2 C_{N,\epsilon}\sqrt N/\epsilon^2$ samples to test all DPPs with probability $0.99$, which completes the proof of Theorem~\ref{thm:upperboundtheorem}.

%% file: LB.tex
\section{Lower bound} \label{section:lowerbound} 
 Next, we establish the lower bound in Theorem~\ref{thm:lowerboundtheorem} for testing DPPs, which implies that the sample complexity of \dpptester\ is tight. In fact, our lower bound is more general: it applies to the problem of testing any subset $\Upsilon$ of the larger class of log-submodular distributions, when $\Upsilon$ includes the uniform measure:
 
\begin{theorem} \label{thm:logsubmodular}
For $\epsilon \leq 0.0005$ and $n \geq 22$, given sample access to a distribution $q$ over subsets of $[n]$, any $(\epsilon, 0.99)$-tester that checks whether $q \in \Upsilon$ or $q$ is $\epsilon$-far from all log-submodular distributions requires $\Omega(\sqrt{N}/\epsilon^2)$ samples. 
\end{theorem}
One may also wish to test if  $q$ is $\epsilon$-far only from  the distributions in $\Upsilon$. A tester for this question, however, will correctly return \reject\ for any $q$ that is $\epsilon$-far from the set of all log-submodular distributions, and can hence distinguish the cases in Theorem~\ref{thm:logsubmodular} too. Hence, the lower bound extends to this question.

Theorem~\ref{thm:lowerboundtheorem} is simply a consequence of Theorem~\ref{thm:logsubmodular}: we may set $\Upsilon$ to be the set of all DPPs, or all $(\alpha,\zeta)$-normal DPPs. Both classes include the uniform distribution over $2^{[n]}$, which is an $(\alpha, \zeta)$-normal DPP with marginal kernel $I/2$, where $I$ is the $n \times n$ identity matrix. By the same argument, the lower bound applies to distinguishing $(\alpha, \zeta)$-normal DPPs from the $\epsilon$-far set of all DPPs.

\textbf{Proof of Theorem~\ref{thm:logsubmodular}.} To prove Theorem~\ref{thm:logsubmodular}, we construct a hard uniformity testing problem that can be decided by our desired tester for $\Upsilon$. In particular, we construct a family $\mathcal{F}$, such that it is hard to distinguish between the uniform measure and a randomly selected distribution $h$ from $\mathcal{F}$.
While the uniform measure is in $\Upsilon$, we will show that $h$ is also far from the set of log-submodular distributions with high probability. Hence, a tester for $\Upsilon$ can, with high probability, correctly decide between $\mathcal{F}$ and the uniform measure.

We obtain $\mathcal{F}$ by randomly perturbing the atom probabilities of the uniform measure over $2^{[n]}$ by $\pm \epsilon'/N$, with $\epsilon' = c \cdot \epsilon$ for a sufficiently large constant $c$ (specified in the appendix). 
More concretely, for every vector $r \in \{\pm 1\}^N$ whose entries are indexed by the subsets $S \subseteq [n]$, we define the distribution $h_r \in \mathcal F$ as  
\begin{equation*}
\forall S \subseteq [n]:\quad h_r(S) \propto \bar h_r(S) = \frac{1 + r_S \epsilon'}{N} \,,
\end{equation*}
where $\bar h_r$ is the corresponding unnormalized measure.

We assume that $h_r$ is selected from $\mathcal F$ uniformly at random, i.e., each entry $r_S$ is a Rademacher random variable independent from the others. In particular, it is known that distinguishing such a random $h_r$ from the uniform distribution requires $\Omega(\sqrt N / \epsilon'^2)$ samples~\cite{DiakonikolasK:2016, Paninski:08}.

To reduce this uniformity testing problem to our testing problem for $\Upsilon$ and obtain the lower bound $\Omega(\sqrt N / \epsilon'^2) = \Omega(\sqrt N / \epsilon^2)$ for the sample complexity of our problem, it remains to prove that $h_r$ is $\epsilon$-far from the class of log-submodular distributions with high probability. Hence, Lemma~\ref{lem:eps-submod} finishes the proof.

\begin{lemma}\label{lem:eps-submod}
  With high probability, a distribution $h_r$ drawn uniformly from $\mathcal{F}$ is $\epsilon$-far from all log-submodular distributions.
\end{lemma}

\textbf{Proof sketch for Lemma~\ref{lem:eps-submod}.} We fix an arbitrary log-submodular distribution $f$ and first show that (1)~the $\ell_1$-distance $\ell_1(f,\bar h_r)$ between $f$ and the unnormalized measure $\bar h_r$ is large with high probability, independent of $f$ (we define the $\ell_1$-distance of general measures the same as for probability measures). Then, (2) we show that if $\ell_1(f,\bar h_r)$ is large, $\ell_1(f, h_r)$ is also large. 

To address (1), we define a family $\SSS_r$ of subsets that, as we prove, satisfies:
\begin{enumerate}
    \item[(P1)]  With high probability, $\SSS_r$ has cardinality at least $N / 64$. 
    \item[(P2)] For every $ S \in \SSS_r$, there is a contribution of at least $\epsilon'/8N$ to $\ell_1(f,\bar h_r)$ from the term $V_S$ defined as
    \begin{align*}
        V_S \coloneqq & \tfrac{1}{2}|\bar h_r(S) - f(S)| + \tfrac{1}{2}|\bar h_r(S \cup \{1\}) - f(S \cup \{1\})| + \\
        & \tfrac{1}{2}|\bar h_r(S \cup \{2\}) - f(S \cup \{2\})| + \tfrac{1}{2}|\bar h_r(S \cup \{1,2\}) - f(S \cup \{1,2\})|.
    \end{align*}
\end{enumerate}
Together, the above properties imply that
$\ell_1(\bar h_r, f) \geq (N/64) \times (\epsilon'/8N) = \epsilon'/512.$ 

We define the important family $\SSS_r$ as
\begin{equation*}
\SSS_r \coloneqq \{ S \subseteq [n] \setminus \{1, 2\}\, |\, r_{\left(S \cup \{ 1 , 2 \}\right)} = 1, \ r_{\left(S \cup \{ 2 \}\right)} = -1, \ r_{\left(S \cup \{ 1 \}\right)} = -1\}.
\label{Sdefinition}
\end{equation*}

Property (P1) follows from a Chernoff bound for the random variables $\mathbbm{1}\{S \in \mathcal S_r\}, \, \forall S \subseteq [n] \setminus \{1,2\}$, where $\mathbbm{1}\{.\}$ is the indicator function.
For proving Property P2, we distinguish two cases, depending on the ratio $f((S \cup \{1,2\}) / f(S \cup \{2\})$. One of these cases relies on the definition of log-submodularity.

Finally, to show that (2) a large $\ell_1(f,\bar{h}_r)$ implies a large $\ell_1(f,h_r)$, we control the normalization constant $\sum_{S \subseteq [n]} \bar h_r(S)$.
The full proof may be found in Appendix~\ref{app:appendixi}.

%% file: tradeoffs.tex
\section{Discussion} \label{sec:tradeoffs}
In this paper, we initiate the study of distribution testing for DPPs. Our lower bound of $\Omega(\sqrt{N}/\epsilon^2)$ shows that, despite the rich mathematical structure of DPPs, testing whether $q$ is a DPP or $\epsilon$-far from it has a sample complexity similar to uniformity testing. This bound extends to related distributions that have gained interest in machine learning, namely log-submodular distributions and strongly Rayleigh measures. Our algorithm \dpptester\ demonstrates that this lower bound is tight for DPPs, via an almost matching upper bound of $\tilde{O}(\sqrt{N}/\epsilon^2)$.

One may wonder what changes when using the moment-based learning algorithm from \cite{urschel2017learning} instead of the learner from Section~\ref{sec:learning}, which yields optimal testing sample complexity. With \cite{urschel2017learning}, we obtain a single estimate $\hat K^{\text{new}}$ for $K^*$, apply a single robust $\chi^2$-$\ell_1$ test against $\Pr[\hat K^{\text{new}}]{.}$, and return its output. The resulting algorithm \dpptesterr\ shows a statistical-computational tradeoff: since it performs only one test, it gains in running time, but its sample complexity is no longer logarithmic in $\alpha$ and $\zeta$, and can be larger than $O(\sqrt N/\epsilon^2)$.
\begin{theorem} \label{thm:dpptester2}
 To test against the class of $(\alpha,\zeta)$-normal DPPs, \dpptesterr\ needs $O\Big(n^4 \log(n)/\epsilon^2 \alpha^2 \zeta^2 + \ell (4/\alpha)^{2\ell} \log(n) + \sqrt{N}/\epsilon^2\Big)$ samples, and runs in time $O(N n^3 + n^6 + mn^2)$, where $m$ is the number of input samples and $\ell$ is the cycle sparsity\footnote{The cycle sparsity of a graph is the smallest $\ell'$ such that the cycles with length at most $\ell'$ constitute a basis for the cycle space of the graph.} of the graph corresponding to the non-zero entries of $K^*$.
\end{theorem}
Assuming a constant cycle sparsity may improve the sample complexity, but our lower bound still applies even with assumptions on cycle sparsity.

While the results in this paper focus on sample complexity for general DPPs, it is an interesting avenue of future work to study whether additional structural assumptions, or a widening to strongly log-concave measures \cite{anariI18,anariII}, can lead to further statistical and computational benefits or tradeoffs.

%% file: appendixA.tex
\section{Proof of the Learning Guarantee} 
\label{app:appendixa}

 In this section, we prove Theorem~\ref{thm:learning}. First, we restate the definition of $(\alpha, \zeta)$-normal DPPs (Definition~\ref{def:normalitydefinition}) below. 
\begin{customdef}{1} \label{def:normalitydefinition2}
  For $\zeta \in [0, 0.5]$ and $\alpha \in [0,1]$, a DPP with marginal kernel $K$ is \emph{$(\alpha, \zeta)$-normal} if: 
  \vspace{-5pt}
\begin{enumerate}\setlength{\itemsep}{0pt}
    \item The eigenvalues of $K$ are in the range $[\zeta, 1 - \zeta]$; and
    \item For $i,j \in [n]: K_{i,j} \neq 0 \Rightarrow |K_{i,j}| \geq \alpha$.
\end{enumerate}
\end{customdef}
Next, we reproduce Theorem~\ref{thm:learning} for convenience.
We set $m = \lceil (\ln(1/\delta) +  1)\sqrt{N}/\epsilon^2\rceil$ to be the number of samples, and use the parameter $\varsigma \coloneqq \lceil 200 n^2\zeta^{-1}  \min\{2\xi/\alpha, \sqrt{\xi/\epsilon} \} \rceil$, with $\xi \coloneqq N^{-\frac{1}{4}} \sqrt{\log(n) + 1}$ below:
\begin{customthm}{3}
Let $q$ be an $(\alpha, \zeta)$-normal DPP distribution with marginal kernel $K^*$. Given the parameters defined above, suppose we have $m$ samples from $q$. 
Then, one can generate a  set $\mathcal{M}$ of DPP distributions with cardinality $|\mathcal{M}| = (2\varsigma+1)^{n^2}$, such that, with probability at least $1 - \delta$, there is a distribution $\tilde p \in \mathcal{M}$ with $\chi^2(q,\tilde p) \leq \epsilon^2/500$. 
\end{customthm}

\begin{proof}[Proof of Theorem~\ref{thm:learning}]
 To prove Theorem~\ref{thm:learning}, first we estimate each entry of the marginal kernel $K^*$ and generate the set $\mathcal M$ of our candidate DPPs, which contains a DPP $\tilde p \in \mathcal M$ whose marginal kernel is close to $K^*$ in the Frobenius distance. Then, we show that that the closeness between the marginal kernels of $\tilde p$ and $q$ implies the desired upper bound in $\chi^2$-distance and $\ell_1$-distance. We start by introducing the initial estimate $\hat K$, which estimates the entries of $K^*$ from our samples. 
\\
\\
\textbf{Estimating entries of $\pmb K^*$:} Note that one can write the entries of the matrix $K^*$ in terms of the marginal probabilities of subsets of size one and two as follows:  
\begin{align}
& \Pr[\mathcal J \sim K^*]{i \in \mathcal J} = \det \left(\left[ K^*_{i,i}\right]\right) =  K^*_{i,i}, \label{rel1} \\   
& \Pr[\mathcal J \sim K^*]{\{ i,j\} \subseteq \mathcal J} = \det\left(\left[
    \begin{array}{cc}
        K^*_{i,i} & K^*_{i,j} \\
        K^*_{j, i} & K^*_{j, j}
    \end{array}
\right]\right)  = K^*_{i,i}K^*_{j,j} - (K^*_{i,j})^2. \label{rel2}
\end{align}
Given the sampled subsets $\{ \mathcal J^{(t)}\}_{t=1}^m$, we can estimate the above marginal probabilities using the number of appearances of every single element and every pair of elements among $\mathcal J^{(1)}, \mathcal J^{(2)}, ..., \mathcal J^{(m)}$.

For each $i \in [n]$, we estimate $K^*_{i,i}$ by the average of the $\ind\{\{i\} \subseteq \mathcal J^{(t)}\}$'s: 
$$\hat{K}_{i,i} \coloneqq \frac{1}{m}\sum_{t=1}^m \ind\{\{i\} \subseteq \mathcal J^{(t)}\}\,.$$
We also denote the averages of the $\ind\{\{i,j\} \subseteq \mathcal J^{(t)}\}$'s by $\hat{u}_{i,j}$. 
$$\hat{u}_{i,j} \coloneqq \frac{1}{m}\sum_{t=1}^m \ind\{\{i,j\} \subseteq \mathcal J^{(t)}\}\,.$$ 
Using the estimates $\hat{u}_{i,j}$, $\hat{K}_{i,i}$, and $\hat{K}_{j,j}$, we can also estimate $(K^*_{i,j})^2$ by the term $\hat{K}_{i,i}\hat{K}_{j,j} - \hat{u}_{i,j}$, based on Equation~\eqref{rel2}. To derive confidence intervals for our estimates, we use the Hoeffding bound and a union bound, which implies that with probability at least $1 - \delta
$: 
\begin{align}
    \forall i \in [n]: \quad &\, \hat{K}_{i,i} \in \left[\Pr[\mathcal J \sim K^*]{i \subseteq \mathcal J} - \xi \epsilon \ , \ \Pr[\mathcal J \sim K^*]{i \subseteq \mathcal J} + \xi \epsilon  \right], \label{firstconfidence} 
    \\
    \forall \{ i,j\}\subseteq [n]: \quad & \, \hat{u}_{i,j} \in \left[\Pr[\mathcal J \sim K^*]{\{ i , j \} \subseteq \mathcal J} - \xi \epsilon\ , \ \Pr[\mathcal J \sim K^*]{\{ i , j \} \subseteq \mathcal J} + \xi\epsilon \right]\,, \label{secondconfidence}
\end{align} 
where $\xi \coloneqq N^{-\frac{1}{4}} \sqrt{\log(n) + 1}$. Note that Equation~\eqref{rel2} does not reveal any information about the sign of $K^*_{i,j}$. However,  we can estimate its magnitude $|K^*_{i,j}|$.  Thus, we consider the following two estimates for $K^*_{i,j}$:
\begin{align}
     \forall \{i,j\} \subseteq [n]:\quad  & \begin{array}{l}
         \hat{K}_{i,j}^{(+)} \coloneqq \sqrt{\max \{ \hat{K}_{i,i}\hat{K}_{j,j} - \hat{u}_{i,j} \ , \ 0 \} }\,, 
         \vspace{2mm}\\
          \hat{K}_{i,j}^{(-)} \coloneqq -\sqrt{\max \{ \hat{K}_{i,i}\hat{K}_{j,j} - \hat{u}_{i,j} \ , \ 0 \} }\,. \label{Kijdefinition}
     \end{array}
\end{align}

Now, let $\hat{K}_{i,j}$ be whichever of $\hat{K}_{i,j}^{(+)}$ or $\hat{K}_{i,j}^{(-)}$ has the same sign as $K^*_{i,j}$. Then, according to Equations~\eqref{firstconfidence}, \eqref{secondconfidence}, and \eqref{Kijdefinition}, we achieve:  
\begin{align*}
\left|\hat{K}_{i,j}^2 - {K^*}_{i,j}^2\right| & \leq \left|K^*_{i,i} K^*_{j,j} - \hat{K}_{i,i}\hat{K}_{j,j}\right| + \left|\Pr[\mathcal J \sim K^*]{\{i,j\}\subseteq  \mathcal J} - \hat{u}_{i,j}\right|
\\ \leq & \max \{ |(K^*_{i,i}+\xi \epsilon)(K^*_{j,j}+\xi\epsilon)-K^*_{i,i}K^*_{j,j}| ,  |(K^*_{i,i}-\xi\epsilon)(K^*_{j,j}-\xi\epsilon)-K^*_{i,i}K^*_{j,j}| \} + \xi\epsilon \\
 \leq & 3\xi\epsilon + {(\xi\epsilon)}^2 \leq 4\xi\epsilon,
\end{align*}
where we used $\xi \epsilon \leq 1$ and that $\forall i,j \in [n]: \, |K^*_{i,j}| \leq 1$. Moreover, using the fact that $\hat{K}_{i,j}$ and $K^*_{i,j}$ have the same sign,
\begin{align*}
    |\hat{K}_{i,j} - K^*_{i,j}|^2 \leq |\hat{K}_{i,j} - K^*_{i,j}\|\hat{K}_{i,j} + K^*_{i,j}| = |\hat{K}_{i,j}^2 - {K^*}_{i,j}^2| \leq 4\xi\epsilon,
\end{align*}
which gives 
\begin{align}
|\hat{K}_{i,j} - K^*_{i,j}| \leq 2\sqrt{\xi\epsilon}. \label{epsilonbound}
\end{align}
On the other hand, we have the lower bound $\alpha$ on the absolute value of the non-zero entries of $K^*$ from the $\alpha$-normality condition~\eqref{def:normalitydefinition2}, so for non-zero $K^*_{i,j}$ we have:
\begin{align}
    |\hat{K}_{i,j} - K^*_{i,j}| \leq \frac{4\xi\epsilon}{|\hat{K}_{i,j} + K^*_{i,j}|} 
    = \frac{4\xi\epsilon}{|\hat{K}_{i,j}| + |K^*_{i,j}|}
    \leq \frac{4\xi\epsilon}{\alpha}.  \label{alphabound}
\end{align}
Combining Equation~\eqref{alphabound} and Equation~\eqref{epsilonbound}, we obtain:
\begin{align}
    |\hat{K}_{i,j} - K^*_{i,j}| \leq 2 \epsilon \min\left\{ \frac{2\xi}{\alpha} , \sqrt{\frac{\xi}{\epsilon}}\right\}.
    \label{mainerrorbound}
\end{align}

Note that by dropping the $\alpha$-normality condition, we still have the bound $|\hat{K}_{i,j} - K^*_{i,j}| \leq 2 \sqrt{\xi \epsilon}$. Hence, the upper bound in Equation~\eqref{mainerrorbound} holds even by setting $\alpha = 0$, which is equivalent to having no $\alpha$-normality for $K^*$. 

\textbf{Generating candidate matrices and DPPs for $\MM$:} Our goal is to eventually bound the $\chi^2$-distance between $q$ and our estimated distribution. To achieve this goal (as we see shortly), it is enough that one estimates each entry of $K^*$ up to an additive error of 
\begin{align} \label{sigmadefinition}
\wp \coloneqq \frac{\epsilon \zeta}{100 n^2}.
\end{align}
In some natural parameter regimes, i.e. when $\epsilon = \tilde{\Omega}(\zeta^{-2} N^{-\frac{1}{4}})$ or $\alpha = \tilde{\Omega}(\zeta^{-1}{N^{-\frac{1}{4}}})$, $\wp$ is larger than the upper bound that we already have in Equation~\eqref{mainerrorbound} and so we can return the distribution of $\hat K$ as our estimate for $q$. However, if this is not the case, we need more candidates to make sure at least one of them is close to $K^*_{i,j}$. Note that $K^*_{i,j}$ is already in the range 
\\ $\left[\hat{K}_{i,j} - 2 \epsilon \min\left\{ {2\xi}/{\alpha} , \sqrt{{\xi}/{\epsilon}}\right\}, \hat{K}_{i,j} + 2 \epsilon \min\left\{ {2\xi}/{\alpha} , \sqrt{{\xi}/{\epsilon}}\right\} \right]$
with high probability. Therefore, we divide this range into 
$\varsigma \coloneqq \lceil 2 \epsilon \min\left\{ {2\xi}/{\alpha} , \sqrt{{\xi}/{\epsilon}}\right\}/\wp \rceil = \lceil 200 n^2\zeta^{-1}  \min\{2\xi/\alpha, \sqrt{\xi/\epsilon}\} \rceil$  intervals of equal length. This way, it is guaranteed that the true $K^*_{i,j}$ is $\wp$-close to one of the midpoints of these intervals (except when $K^*_{i,j}$ is zero which we handle separately). As discussed, this partitioning (in the literature of learning theory, this technique is called \textit{bracketing}) allows the algorithm to achieve the optimal sample complexity.

Now, we claim that there are $2 \varsigma + 1$ candidates for $K^*_{i,j}$. This number comes from the fact that we do not know whether $\hat{K}_{i,j}$ is equal to $\hat{K}_{i,j}^{(+)}$ or $\hat{K}_{i,j}^{(-)}$ a priori. Thus, each option provides $\varsigma$ candidates.
Also, we have to consider the case $K^*_{i,j} = 0$ separately because the lower bound $\alpha$ only holds for non-zero entries $K^*_{i,j}$. By considering all the combinations of candidates for each entry, we obtain a set $M$ of matrices. Since each entry has a $\wp$-close candidate, there exists a matrix $\tilde{K} \in M$ such that all of its entries are $\wp$-close to the true kernel matrix $K^*$. Therefore, this matrix is $(n\wp)$-close to $K^*$ in the Frobenius distance. As we discussed in section~\ref{sec:learning}, we project each $K \in M$ onto the set of valid marginal kernels and consider the set of candidate distributions $\MM \coloneqq \{ \Pr[\Pi(K)]{.} | \, K \in M \}$. The projection is with respect to the Frobenius distance between matrices, and it is easy to see that computing $\Pi(K)$ is equivalent to rounding up the eigenvalues of $K$ that are negative to zero, and rounding down the ones that are greater than one to one. Now for the DPP distribution $\tilde p = \Pr[\Pi(\tilde K)]{.} \in \MM$, we prove the following claims:

\begin{enumerate}[start=1,label={(C\arabic*)}]
    \item The kernels $\Pi(\tilde K)$ and $K^*$ are close in operator norm:
    $$\|\Pi(\tilde{K}) - K^* \|_2 \leq \frac{\epsilon \zeta}{100 n}.$$ \label{claim:first} 
    \item 
    The singular values of $\Pi(\tilde K)$ are in the range $[99\zeta/100, 1 - 99\zeta/100]$.  \label{claim:second}
\end{enumerate}

For the first claim~\ref{claim:first}, it is enough to write
\begin{align}
\| \Pi(\tilde{K}) - K^* \|_2 \leq \| \Pi(\tilde{K}) - K^* \|_F =
\| \Pi(\tilde{K}) - \Pi(K^*) \|_F \leq \| \tilde{K} - K^* \|_F \leq  n\wp = \frac{\epsilon \zeta}{100 n}.
\label{eq:frob2}
\end{align}
where $\|.\|_2$ and $\|.\|_F$ refer to matrix operator norm and Frobenius norm respectively. The first inequality holds because the spectral norm is bounded by the Frobenius norm, the first equality follows from the fact that $K^*$ is a valid marginal kernel, and the second inequality is because of the contraction property of projection.

Next, we prove the second claim~\ref{claim:second}. Using the variational characterization of the Operator norm and noting the fact that $ \Pi(\tilde{K}) - K^*$ is symmetric (thus its singular values are the absolute values of its eigenvalues), we have $$\| \Pi(\tilde{K}) - K^* \|_2 = \max_{v,\|v\|_2 = 1} |v^T(\Pi(\tilde{K}) - K^* )v|.$$
Combining this with Equation~\eqref{eq:frob2} then implies the following for every normalized vector $\|v\|_2 = 1$:
\begin{align}
 -\frac{\epsilon \zeta}{100n} \leq v^T(\Pi(\tilde{K}) - K^*)v \leq \frac{\epsilon \zeta}{100n} \label{eq:variational1}
\end{align}
Since $\Pr[K^*]{.}$ is $\zeta$-normal due to our assumption, we also have
\begin{align}
\zeta \leq v^T K^* v \leq 1 - \zeta. \label{eq:variational2}    
\end{align}
Combining Inequalities~\eqref{eq:variational1} and~\eqref{eq:variational2} yields
$$v^T \Pi(\tilde{K}) v \geq \zeta -\frac{\epsilon \zeta}{100n} \geq \zeta -\frac{ \zeta}{100} = \frac{99\zeta}{100},$$
and similarly 
$$v^T \Pi(\tilde{K}) v \leq 1 - \zeta + \frac{\epsilon \zeta}{100n} \leq 1 - \frac{99\zeta}{100},$$
for any arbitrary normalized vector $v$. Finally, using the variational characterization of the smallest and largest eigenvalues, we obtain that all eigenvalues of $\Pi(\tilde K)$ are in the range $[99\zeta/100, 1-99\zeta/100]$. Note that the singular values of $\Pi(\tilde{K})$ are the absolute values of its eigenvalues, simply because $\Pi(\tilde{K})$ is symmetric, which completes the proof of the second claim~\ref{claim:second}. We use these claims~\ref{claim:first},~\ref{claim:second} in the next part.
\\
\\
\noindent\textbf{Closeness in parameter space implies closeness of the distributions: }
In this part of the proof, we show that closeness between $K^*$ and $\Pi(\tilde{K})$ in operator norm ensures the closeness of the distributions $q$ and $\tilde p$ with respect to the $\chi^2$-distance and $\ell_1$-distance. This result is based on the following Lemma, whose proof we defer to the end of this section.
\begin{lemma}\label{lem:perturbation}
For arbitrary symmetric matrices $B$ and $E$, we have
$$\Big| |\det(B+E)| - |\det(B)| \Big| \leq |\det(B)| \frac{n\|E\|_2}{\sigma_n(B)} \left( \frac{\|E\|_2}{\sigma_n(B)} + 1 \right)^{n-1},$$
where $\sigma_n(B)$ is the smallest singular value of $B$.
\end{lemma}

Now consider an arbitrary set $J \subseteq [n]$ and its complement $\bar J$. 
Recall that Equation~\eqref{eq:atomprobabilities} gives:
\begin{align*}
    \tilde p(J) = |\det(\Pi(\tilde K) - I_{\bar{J}})| \, , \,
    q(J) = |\det(K^* - I_{\bar J})|.
\end{align*}
Therefore, setting $B \coloneqq \Pi(\tilde K) - I_{\bar J}$ and $E \coloneqq K^* - \Pi(\tilde{K})$ in Lemma~\ref{lem:perturbation}, we can upper bound $|q(J) - \tilde p(J)|$ as
\begin{align}
|q(J) - \tilde p(J)| \leq \tilde p(J) \frac{n\| E\|_2}{\sigma_n(B)} \left(\frac{\| E\|_2}{\sigma_n(B)} + 1\right)^{n-1}. \label{eq:initialatombound}
\end{align}

Furthermore, from the second claim~\ref{claim:second} of the previous part, the singular values of $\Pi(\tilde K)$ are in the range $[99\zeta/100, 1-99\zeta/100]$, which means the kernel matrix $\Pi(\tilde K)$ satisfies the condition of Lemma~\ref{smallest-singular-value-lemma}. Therefore, from Lemma~\ref{smallest-singular-value-lemma}, the smallest singular value of $B$ is lower bounded as
 $$\sigma_n(B) \geq \frac{99\zeta/100(1- 99\zeta/100)}{\sqrt 2} \geq \frac{99\zeta}{200\sqrt 2},$$
 where we used $1 - 99\zeta/100 > 1/2$.
 Combining this with the first claim~\ref{claim:first} of the previous part implies $$\frac{\| E\|_2}{\sigma_n(B)} \leq \frac{2 \sqrt 2 \epsilon}{99 n}.$$
 Hence, Equation~\eqref{eq:initialatombound} gives: 
\begin{align}
    |q(J) - \tilde p(J)| \leq \tilde p(J) \frac{2\sqrt2\epsilon}{99} \left(\frac{2\sqrt2\epsilon}{99n} + 1 \right)^{n-1} \leq \frac{\epsilon}{25} \tilde p(J), \label{initialupperbound}
\end{align}
where the last inequality follows from
$$\left( \frac{ 2\sqrt 2 \epsilon }{99 n} + 1 \right)^{n-1} < \left(\frac{ 2\sqrt 2 }{99 n} + 1 \right)^{n-1} < \frac{99}{50\sqrt 2} \,\,\, \forall n \in \mathbb N.$$
Note that $J \subseteq [n]$ is arbitrary, so Equation~\eqref{initialupperbound} finally yields the desired bound on the $\ell_1$-distance and $\chi^2$-distance between $q$ and $\tilde p$: 
\begin{align*}
    &\ell_1(q,\tilde p) = \sum_{J \subseteq [n]} |q(J) - \tilde p(J)| \leq \sum_{J \subseteq [n]} \frac{\epsilon}{25} \tilde p(J) = \frac{\epsilon}{25}, \\
    &\chi^2(q,\tilde p) = \sum_{J \subseteq [n]} \frac{(q(J) - \tilde p(J))^2}{\tilde p(J)} < \sum_{J \subseteq [n]} \frac{\epsilon^2}{500} \tilde p(J) = \frac{\epsilon^2}{500}.
 \end{align*}
\end{proof}

\begin{proof}[Proof of Lemma~\ref{lem:perturbation}]
 Let $\sigma_1\geq\cdots\geq\sigma_n$ be the singular values of $B$. For every $0 \leq k \leq n$, we denote $s_k$ the $k$th elementary symmetric function on the singular values of $B$, i.e.
$$s_0 = 1, \ \forall \, 1 \leq k \leq n: \, s_k = \sum_{1 \leq i_1 < \ldots < i_k \leq n} \sigma_{i_1}\ldots\sigma_{i_k},$$
 Note since $B$ is symmetric, the singular values are the absolute values of the eigenvalues, which implies the relation $|\det(B)| = \sigma_1 \cdots \sigma_n$.

Now Corollary 2.7 of~\cite{ipsen2008perturbation} states the following determinant's perturbation inequality:
$$\Big|\det(B+E) - \det(B)\Big| \leq \sum_{i=1}^n s_{n-i}\|E\|_2^i.$$

From this, we can derive
\begin{align*}
   \Big||\det(B + E)| - |\det(B)|\Big| \leq 
   \Big|\det(B + E) - \det(B)\Big| & \leq \sum_{i=1}^n s_{n-i} \| E \|_2^i  \\ 
     & = |\det(B)|  \sum_{i=1}^n \frac{s_{n-i}}{\sigma_1 \ldots \sigma_n} \| E\|_2^i,
\end{align*}
where in the last equality, we multiplied and divided the sum by $|\det(B)|$. Moving forward, we bound $s_{n-i}$  by $\binom{n}{i} \sigma_1\cdots\sigma_{n-i}$:
\begin{align*}
    \Big||\det(B + E)| - |\det(B)|\Big| & \leq  |\det(B)| \sum_{i=1}^n \binom{n}{i} \frac{\sigma_1\ldots\sigma_{n-i}}{\sigma_1\ldots\sigma_n} \| E\|_2^i \\
     & =  |\det(B)| \sum_{i=1}^n \binom{n}{i} \frac{1}{\sigma_{n-i+1}\ldots\sigma_n} \| E\|_2^i \nonumber \\
     & \leq  |\det(B)| \sum_{i=1}^n \binom{n}{i} \left(\frac{\| E\|_2}{\sigma_n}\right)^i \\
      & \leq |\det(B)|\,n \sum_{i=1}^n \binom{n-1}{i-1} \left(\frac{\| E\|_2}{\sigma_n}\right)^{i} \nonumber \\
      & =  |\det(B)| \frac{n\,\| E\|_2}{\sigma_n} \sum_{i=0}^{n-1} \binom{n-1}{i} \left(\frac{\| E\|_2}{\sigma_n}\right)^{i} \\
      & = |\det(B)| \frac{n\| E\|_2}{\sigma_n} \left(\frac{\| E\|_2}{\sigma_n} + 1\right)^{n-1}. 
\end{align*}
\end{proof}

%% file: appendixB.tex
\section{Uniform Lower Bound on the Smallest 
Singular Value of $K - I_{\bar J}$}
\label{app:appendixb}

In this section, we prove Lemma~\ref{smallest-singular-value-lemma}: given a marginal kernel $K$ whose eigenvalues are in the range $[\zeta, 1-\zeta]$, we prove the uniform lower bound $\zeta(1-\zeta)/\sqrt 2$ on the singular values of the family of matrices $\{K-I_{\bar J}\}_{J \subseteq [n]}$. This Lemma is used in the proof of Theorem~\ref{thm:learning} and enables us to control the distances between the atom probabilities of $\Pr[K]{.}$ and $\Pr[\Pi(\tilde K)]{.}$.
\begin{proof}[Proof of Lemma~\ref{smallest-singular-value-lemma}]
Let $\lambda_1 \geq ... \geq \lambda_n$ be the eigenvalues of $K$ and  $v_1 , ... , v_n$ be an orthonormal set of their corresponding eigenvectors. 
We fix a subset $J \subseteq [n]$ and lower bound the smallest singular value of $K - I_{\bar J}$ based on its variational characterization:
\begin{align}
\sigma_n(K - I_{\bar J}) = \min_{\|v\|_2 = 1} \sqrt{v^T(K - I_{\bar J})^2 v}. \label{eq:variationalformulation}
\end{align}
Given a normalized vector $v$: $\|v\|_2 = 1$, we represent $v$ in the basis $\{v_i\}_{i=1}^n$ as $v = \sum_{i=1}^n \alpha_i v_i$. Because $\{v_i\}_{i=1}^n$ is orthonormal, we have
$$1 = \|v\|^2 = \sum_{i=1}^n \alpha_i^2 \|v_i\|^2 = \sum_{i=1}^n \alpha_i^2.$$
Now we can express $v^T(K - I_{\bar J})^2 v$ as:
\begin{align*}
v^T (K - I_{\bar J})^2 v & = 
\left(\sum_{i=1}^n \alpha_i v_i \right)^T (K - I_{\bar J})^2 \left(\sum_{i=1}^n \alpha_i v_i \right) \\
& = \sum_{1 \leq i,j \leq n} \alpha_i \alpha_j v_i^T (K - I_{\bar J})^2 v_j \\
& = \sum_{1 \leq i,j \leq n} \alpha_i \alpha_j v_i^T K^2 v_j + \sum_{1 \leq i,j \leq n} \alpha_i \alpha_j \big( v_i^T I_{\bar J}^2 v_j - v_i^T K I_{\bar J} v_j - v_i^T I_{\bar J} K  v_j \big).
\end{align*}
Observe that ${v_i}^T K^2 v_i = \lambda_i^2 \|v_i\|^2 = \lambda_i^2$ and ${v_i}^T K^2 v_j = \lambda_i \lambda_j {v_i}^T v_j = 0$ for $i \neq j$. We define some additional notation here: For any subset $J \subseteq [n]$, let $(v_i)_J$  be the restriction of $v_i$ into support $J$. We also denote the inner product of the vectors $v_i$ and $v_j$ restricted to $J$ by $\big\langle v_i , v_j \big\rangle_J$. Using these notations, we can further simplify the terms ${v_i}^T I_{\bar J}^2 v_j$, ${v_i}^T K I_{\bar{J}} v_j$ and ${v_i}^T I_{\bar J} K v_j$ to $\big\langle v_i , v_j\big\rangle_{\bar J}$, $\lambda_i \big\langle v_i , v_j\big\rangle_{\bar J}$, and $\lambda_j \big\langle v_i , v_j\big\rangle_{\bar J}$ respectively. Substituting them above results in   
\begin{align*}
v^T (K - I_{\bar J})^2 v = & \sum_{i=1}^n \alpha_i^2 \lambda_i^2 + \sum_{1 \leq i,j \leq n} (1 - \lambda_i - \lambda_j) \alpha_i \alpha_j \big\langle v_i , v_j\big\rangle_{\bar J}  \\
 = & \sum_{i=1}^n \alpha_i^2 \lambda_i^2 - \sum_{1 \leq i,j \leq n} \alpha_i \alpha_j \lambda_i \lambda_j \big\langle v_i , v_j\big\rangle_{\bar J} + \sum_{1 \leq i,j \leq n} \alpha_i \alpha_j (1 - \lambda_i) (1 - \lambda_j) \big\langle v_i , v_j\big\rangle_{\bar J} \\
 \end{align*}
 where the last equality simply follows from the Equation $(1 - \lambda_i)(1 - \lambda_j) = 1 - \lambda_i - \lambda_j + \lambda_i\lambda_j$. Now substituting $\big\langle v_i, v_j \big\rangle_{\bar J}$ by $\big\langle v_i, v_j \big\rangle - \big\langle v_i, v_j \big\rangle_{J}$ in the second term above, we obtain
 \begin{align*}
 v^T & (K - I_{\bar J})^2 v \\
 = & \sum_{i=1}^n \alpha_i^2 \lambda_i^2 - \sum_{1 \leq i,j \leq n} \alpha_i \alpha_j \lambda_i \lambda_j \big\langle v_i , v_j\big\rangle \\
  + & \sum_{1 \leq i,j \leq n} \alpha_i \alpha_j \lambda_i \lambda_j \big\langle v_i, v_j \big\rangle_J + \sum_{1 \leq i,j \leq n} \alpha_i \alpha_j (1 - \lambda_i) (1 - \lambda_j) \big\langle v_i , v_j\big\rangle_{\bar J} \\
 = & \sum_{i=1}^n \alpha_i^2 \lambda_i^2 - \sum_{i=1}^n \alpha_i^2 \lambda_i^2 + \sum_{1 \leq i,j \leq n} \alpha_i \alpha_j \lambda_i \lambda_j \big\langle v_i, v_j \big\rangle_J + \sum_{1 \leq i,j \leq n} \alpha_i \alpha_j (1 - \lambda_i) (1 - \lambda_j) \big\langle v_i , v_j\big\rangle_{\bar J} \\
  = & \Big\|\sum_{i=1}^n \alpha_i \lambda_i {(v_i)}_{J}\Big\|^2 + \Big\|\sum_{i=1}^n \alpha_i (1 - \lambda_i) {(v_i)}_{\bar J}\Big\|^2.   \label{endequation} \numberthis
 \end{align*}
 
 Hence, it suffices to derive a lower bound on $\Big\|\sum_{i=1}^n \alpha_i \lambda_i {(v_i)}_{J}\Big\|^2 + \Big\|\sum_{i=1}^n \alpha_i (1 - \lambda_i) {(v_i)}_{\bar J}\Big\|^2$ independent from $J$. To this end, we define the column vectors $w_1 = \Big( \alpha_i \lambda_i \Big)_{i=1}^n$, $w_2 = \Big( \alpha_i (1 - \lambda_i) \Big)_{i=1}^n$. Furthermore, define $R \coloneqq \bigg( v_1 \bigg\vert  \dots \bigg\vert v_n \bigg)$ as the matrix with $v_i$ as its $i$th column, and let ${v'_1}^T, ... ,{v'_n}^T$ be the rows of $R$. Because $\{ v_i\}_{i=1}^n$ is an orthonormal set, $R$ is a unitary matrix, so $\{v'_j\}_{j=1}^n$ is also an orthonormal set. Next, let $V$ and $V^T$ be the subspaces spanned by the set of vectors $\{v'_j\}_{j \in \bar J}$ and $\{v'_j\}_{j \in J}$ respectively. Because $\{v'_j\}_{j=1}^n$ is an orthonormal set, the subspaces $V$ and $V^{\perp}$ are orthogonal to each other. Let $\nu_1 = \sum_{j \in \bar J} ({v'_j}^T w_1)v'_j$ and ${\nu_1}^\perp = \sum_{j \in J} ({v'_j}^T w_1)v'_j$ be the projections of $w_1$ onto $V$ and $V^\perp$ respectively. Similarly, define $\nu_2 = \sum_{j \in \bar J} ({v'_j}^Tw_2)v'_j$ and ${\nu_2}^\perp = \sum_{j \in J} ({v'_j}^Tw_2)v'_j$ as the projections of $w_2$ onto $V$ and $V^\perp$. Now by decomposing $w_1$ on $V$ and $V^\perp$, we can write 
\begin{align*}
    & w_1 = \nu_1 + {\nu_1}^\perp.
\end{align*}
Similarly, we have
$$w_2 = \nu_2 + {\nu_2}^\perp.$$
Moreover, from the orthonormality of $v'_1 , ... , v'_n$, we obtain
\begin{align*}
    \|{\nu_1}^\perp\|^2 = & \Big\|\sum_{j \in J} ({v'_j}^T w_1) v'_j\Big\|^2 = \sum_{j \in J} ({v'_j}^T w_1)^2 = \sum_{j \in J} (\sum_{i=1}^n R_{j,i} {(w_1)}_i)^2 \\
    = &\sum_{j \in J} (\sum_{i=1}^n {(v_i)}_j {(w_1)}_i)^2  = \Big\|\sum_{i=1}^n \alpha_i \lambda_i {(v_i)}_{J}\Big\|^2. 
\end{align*}
Similarly, one obtains
\begin{align*}
    \|\nu_2\|^2 = \Big\|\sum_{j \notin J} ({v'_j}^T w_2) v'_j\Big\|^2 = \Big\|\sum_{i=1}^n \alpha_i (1 - \lambda_i) {(v_i)}_{\bar J}\Big\|^2.
\end{align*}
Combining the last two equations with Equation~\eqref{endequation}, we obtain
\begin{align}
v^T (K - I_{\bar J})^2 v = \|\nu_2\|^2 + \|{\nu_1}^{\perp}\|^2. \label{mainequation}
\end{align}
Now, it suffices to bound $\|\nu_2\|^2 + \|{\nu_1}^{\perp}\|^2$. Note that
\begin{align*}
&\|w_1\|^2 = \sum_{i=1}^n \alpha_i^2 \lambda_i^2 \leq \sum \alpha_i^2 = 1, \\
&\|w_2\|^2 = \sum_{i=1}^n \alpha_i^2 (1-\lambda_i)^2 \leq \sum \alpha_i^2 = 1.
\end{align*}
which implies $\|\nu_1\| , \|\nu_2\|, \|{\nu_1}^{\perp}\|, \|{\nu_2}^{\perp}\| \leq 1$. Moreover, the condition $\zeta \leq \lambda_i \leq 1 - \zeta$ implies $\lambda_i(1-\lambda_i) \geq \zeta (1 - \zeta)$. Therefore, on one hand, we get
\begin{align}
\big\langle w_1 , w_2 \big\rangle = \sum_{i=1}^n \lambda_i (1 - \lambda_i) \alpha_i^2 \geq \zeta(1 - \zeta) \sum_{i=1}^n \alpha_i^2 = \zeta(1 - \zeta). \numberthis \label{eq:aval}
\end{align}
On the other hand,
\begin{align*}
\big\langle w_1 , w_2 \big\rangle & = \big\langle \nu_1 +  {\nu_1}^\perp , \nu_2 + {\nu_2}^{\perp} \big\rangle =  \big\langle \nu_1 , \nu_2 \big\rangle +  \big\langle {\nu_1}^\perp , {\nu_2}^{\perp} \big\rangle \\
& \leq \|\nu_1\|\|\nu_2\| + \|{\nu_1}^\perp\| \|{\nu_2}^{\perp}\|
 \leq \|\nu_2\| + \|{\nu_1}^\perp\| \\
& \leq \sqrt{2(\|\nu_2\|^2 + \|{\nu_1}^\perp\|^2)} = \sqrt{2v^T (K-I_{\bar J})^2v}. \numberthis \label{eq:dovom}
\end{align*}
where the last equality follows from Equation~\eqref{mainequation}. Combining Equations~\eqref{eq:aval} and~\eqref{eq:dovom}, we conclude $v^T(K - I_{\bar J})^2v \geq \zeta^2(1 - \zeta)^2/2$. Recall that $v$ is an arbitrary normalized vector, and $J$ is an arbitrary subset of $[n]$, so the variational characterization of $\sigma_n$ in Equation~\eqref{eq:variationalformulation} yields the desired lower bound $\sigma_n(K - I_{\bar J}) \geq \zeta (1 - \zeta)/\sqrt 2$ for every $J \subseteq [n]$.
\end{proof}

%% file: appendixI.tex
\section{Lower Bound for Testing Log-submodular Distributions} \label{app:appendixi}
In this section, we rigorously prove Lemma~\ref{lem:eps-submod}, which in turn completes the proof of Theorem~\ref{thm:logsubmodular}. We assume that $\epsilon'$, $\mathcal F$, $h_r$ and $\bar h_r$ are defined as in Section~\ref{section:lowerbound}.

\begin{proof}[Detailed Proof of Lemma~\ref{lem:eps-submod}] Given $\epsilon' \leq \frac{2}{3}$ and a log-submodular distribution $f$, we first show that the $\ell_1$-distance between $f$ and the unnormalized measure $\bar h_r$ is large with high probability independent of $f$ (we define the $\ell_1$-distance of general measures the same as for probability measures.)
To this end, we define the following family of subsets based on $h_r$, that is random:
\begin{align}
\SSS_r \coloneqq \{ S \subseteq [n] \setminus \{1, 2\}\, |\,\, r_{\left(S \cup \{ 1 , 2 \}\right)} = 1, \ r_{\left(S \cup \{ 2 \}\right)} = -1, \ r_{\left(S \cup \{ 1 \}\right)} = -1\}.
\label{Sdefinition}
\end{align}
We prove that $\mathcal S_r$ has the following properties: 
\begin{enumerate}[start=1,label={(P\arabic*)}]
    \item \label{prop:first} With high probability, the cardinality of $\SSS_r$ is at least $N / 64$. 
    
    \item \label{prop:second} For every $ S \in \SSS_r$, there is a contribution of at least $\epsilon'/8N$ to the $\ell_1$-distance between $\bar h_r$ and $f$ from the term $V_S$ defined as
    \begin{align*}
        V_S \coloneqq & \frac{1}{2}|\bar h_r(S) - f(S)| + \frac{1}{2}|\bar h_r(S \cup \{1\}) - f(S \cup \{1\})| + \\
        & \frac{1}{2}|\bar h_r(S \cup \{2\}) - f(S \cup \{2\})| + \frac{1}{2}|\bar h_r(S \cup \{1,2\}) - f(S \cup \{1,2\})|.
    \end{align*}

\end{enumerate}
Note that based on these two properties, one can simply derive
\begin{equation}
\ell_1(\bar h_r, f) \geq \frac{N}{64} \times \frac{\epsilon'}{8N} = \frac{\epsilon'}{512} \label{eq:unnormalized}
\end{equation}
with high probability. 

To show that the event $\mathcal Q_1 \coloneqq \{|\SSS_r| \geq N/64\}$ happens with high probability for the first property~\ref{prop:first}, we use a Chernoff bound for the random variables $\mathbbm{1}\{S \in \mathcal S_r\}, \, \forall S \subseteq [n] \setminus \{1,2\}$, where $\mathbbm{1}\{.\}$ is the indicator function. Clearly, for  each $S \subseteq [n] \setminus \{1,2\}$, we have $\mathbb{E}[\mathbbm{1}\{S \in \SSS_r\}] = \Pr{S \in \SSS_r} = 1/8$, and $\mathbb{E}[\vert\SSS_r\vert] = N/32$. Therefore,
$$
\Pr{\mathcal Q_1^c} = \Pr{\sum_{S \in [n] \setminus \{1,2\}} \mathbbm{1}\{S \in \SSS_r\} <  \left(1 - \frac{1}{2}\right) \mathbb{E}[|\SSS_r|]} \leq \exp \left(- 0.5 \frac{N}{32}(\frac{1}{2})^2 \right) = \exp\left(-\frac{N}{256}\right).
$$
We conclude for $n \geq n_1 = 11$, $\mathcal Q_1$ happens with probability at least $0.995$. 

We now prove the second property~\ref{prop:second}. Fix a set $S \in \mathcal S_r$ and
define the constant $\rho \coloneqq \frac{1+\epsilon'}{1 - 3\epsilon'/4}$. To prove $V_S \geq \frac{\epsilon'}{8N}$, we consider two cases:

\paragraph{Case 1:} $
\frac{f(S \cup \{1,2\})}{f(S \cup \{2\})} \leq \rho \label{dovom}
$ \\
Here, we formalize a helper inequality in the following Lemma, and prove it at the end of this section.

\begin{lemma} \label{lem:helper}
For $a,b \geq 0$, the condition $\frac{a}{b} \leq \rho$ implies $|1+\epsilon' - a| + |1-\epsilon' - b| \geq \frac{\epsilon'}{4}.$ 
\end{lemma}

Now from $S \in \mathcal S_r$, we get $\bar h_r(S \cup \{1,2\}) = \frac{1+\epsilon'}{N}$ and $\bar h_r(S \cup \{2\}) = \frac{1-\epsilon'}{N}$. Hence,
\begin{align*}
V_S & \geq \frac{1}{2}| \bar h_r(S \cup \{ 1 , 2\}) - f (S \cup  \{ 1 , 2 \}) | + \frac{1}{2}| \bar h_r(S \cup \{ 2\}) - f(S \cup  \{ 2 \})| 
\\
& = \frac{1}{2}\Big| \frac{1+\epsilon'}{N} - f (S \cup  \{ 1 , 2 \}) \Big| + \frac{1}{2}\Big| \frac{1-\epsilon'}{N} - f(S \cup  \{ 2 \})\Big| \geq \frac{\epsilon'}{8N},
\end{align*}
where the last inequality follows from Lemma~\ref{lem:helper}, by setting $ a = Nf(S \cup \{1,2\}), \, b = Nf(S \cup \{2\})$.

\paragraph{Case 2:} $\frac{f(S \cup \{1,2\})}{f(S \cup \{2\})} > \rho$ \\
In this case, the log-submodularity property allows us to write
$$
\log(f(S \cup \{1\})) - \log(f(S)) \geq \log(f(S \cup \{1,2\})) - \log(f(S \cup \{2\})) > \log(\rho), 
$$
or equivalently
\begin{equation}
\frac{f(S \cup \{1\})}{f(S)} > \rho = \frac{1+\epsilon'}{1-3\epsilon'/4}.  \label{charom}
\end{equation}
Note that from $S \in \mathcal S_r$, we have $\bar h_r(S \cup \{ 1 \}) = \frac{1 - \epsilon'}{N}$. If $f(S \cup \{1\})$ is larger than $\frac{1-3\epsilon'/4}{N}$, then 
$$V_S \geq \frac{1}{2}|\bar h_r(S \cup \{1\}) - f(S \cup \{1\})| > \frac{1}{2}\Big(\frac{1-3\epsilon'/4}{N} - \frac{1-\epsilon'}{N}\Big) = \frac{\epsilon'}{8N}$$
and we are done. Otherwise, we have 
 $
 f(S \cup \{1\}) \leq \frac{1-3\epsilon'/4}{N}.
 $
 Combining this with Equation~\eqref{charom} gives:
 \begin{equation*}
\,\,f(S) \leq \rho^{-1} f(S \cup \{1\}) \leq  \frac{1-3\epsilon'/4}{1 + \epsilon'} \times \frac{1-3\epsilon'/4}{N}  \leq \frac{1 - \epsilon'}{N} - \frac{\epsilon'}{4N},
 \end{equation*}
 where the last inequality follows from the condition $\epsilon' \leq \frac{2}{3}$. Finally, we obtain
 \begin{equation*}
 V_S \geq \frac{1}{2}|\bar h_r(S) - f(S)| \geq \frac{1}{2}\Big(\frac{1-\epsilon'}{N} - (\frac{1-\epsilon'}{N} - \frac{\epsilon'}{4N} )\Big) = \frac{\epsilon'}{8N},
 \end{equation*}
which completes the proof for the second property~\ref{prop:second}. Therefore, under the occurrence of $\mathcal Q_1$, we conclude from Equation~\eqref{eq:unnormalized} that $\ell_1(\bar h_r, f) \geq \frac{\epsilon'}{512}$. To show the $\ell_1$-distance between $h_r$ and $f$ is also large, we control the normalization constant $L_r \coloneqq \sum_{S \subseteq [n]} \bar h_r(S)$. Define the event $\mathcal Q_2 \coloneqq \{ 1 - \frac{4 \epsilon'}{\sqrt N} \leq L_r \leq 1 + \frac{4 \epsilon'}{\sqrt N} \}$
. A simple Hoeffding bound for the random variables $\frac{1 + r_S \epsilon'}{N} ,\, \forall S \subseteq [n],$ implies that $\mathcal Q_2$ happens with probability at least $0.995$. Now under the occurrence of $\mathcal Q_1 \cap \mathcal Q_2$ and assuming $n \geq n_2 = 22$, we can write:
\begin{align*}
    2\ell_1(h_r, f) = & \sum_{S \subseteq [n]} |h_r(S) - f(S)| = \sum_{S \subseteq [n]} \Big|\frac{\bar h_r(S)}{L_r} - f(S)\Big| \\
    \geq & \sum_{S \subseteq [n]} |\bar h_r(S) - f(S)| - \sum_{S \subseteq [n]} \bar h_r(S)\Big|\frac{1 - L_r}{L_r}\Big| \\
     \geq & \frac{\epsilon'}{256} - \frac{4 \epsilon'}{ L_r \sqrt N} \sum_{S \subseteq [n]} \bar h_r(S)
     \geq  \epsilon' (\frac{1}{256} - \frac{4}{\sqrt N}) \geq \epsilon'(\frac{1}{256} - \frac{1}{512}) = \frac{c \epsilon}{512}.
\end{align*}
A union bound for the events $Q_1^c$ and $Q_2^c$ implies that $\mathcal Q_1 \cap \mathcal Q_2$ happens with probability at least $0.99$. Note that $\mathcal Q_1$ and $\mathcal Q_2$ does not depend on $f$. Setting $c = 1024$, we conclude that with probability at least $0.99$, $\ell_1(h_r, f) \geq \epsilon$ for any log-submodular distribution $f$, given that $\epsilon = \epsilon'/c \leq \frac{2}{3 \times 1024}$ and $n \geq \max \{n_1, n_2\} = 22$, which completes the proof of Lemma~\ref{lem:eps-submod}.
\end{proof}

\begin{proof}[Proof of Lemma~\ref{lem:helper}]
Here, we prove Lemma~\ref{lem:helper}, which we used above. First note that if $b \geq a$, then clearly $|b - (1-\epsilon')| + |a - (1+\epsilon')| \geq 2\epsilon' > \frac{\epsilon'}{4}$. So we assume $b < a$. 

Now define $t \coloneqq a - (1 + \epsilon')$, so that $a = 1 + \epsilon' + t$. Then, we can write 
\begin{align*}
|b - (1-\epsilon')| + |a - (1+\epsilon')| & = |\frac{b}{a}(1+\epsilon' +t) - (1-\epsilon')| + |t| \\
& \geq |\frac{b}{a}(1+\epsilon') - (1-\epsilon')| - |\frac{b}{a}t| + |t| \\
& = |\frac{b}{a}(1+\epsilon') - (1-\epsilon')| + (1-\frac{b}{a})|t|.
\end{align*}

The condition $\frac{a}{b} \leq \rho$ implies $\frac{b}{a}(1+\epsilon') \geq 1-\frac{3\epsilon'}{4}$. Therefore
$$|b - (1-\epsilon')| + |a - (1+\epsilon')| \geq \frac{\epsilon'}{4} + (1 - \frac{b}{a})|t| \geq \frac{\epsilon'}{4}.$$
where the last inequality follows from the fact that $1 - \frac{b}{a} > 0$.
\end{proof}

%% file: appendixH.tex
\section{Coupling DPPs}\label{app:appendixh}
In this section, we fully introduce and prove the coupling argument of Lemma~\ref{lemma:coupling}. Given a value $0 < z \leq 0.5$ and a DPP whose marginal kernel has eigenvalues that are outside the range $[z,1-z]$, the goal is to couple it with another DPP, which has a marginal kernel with all eigenvalues in $[z,1-z]$, such that the data sets generated from these two DPPs are equal with high probability.

\begin{proof}[Proof of Lemma~\ref{lemma:coupling}]
Let $V$ be an orthonormal set of the eigenvectors of $K$. For each $v \in V$, let $\lambda_v$ be its corresponding eigenvalue. To introduce our coupling, we need to define the class of \textit{elementary DPPs}~\cite{kulesza2012determinantal}. A DPP is called \textit{elementary} if the eigenvalues of its marginal kernel are either zero or one. For each subset $V' \subseteq V$ of the eigenvectors of $K$, we consider the elementary DPP $\Pr[K^{V'}]{.}$ with marginal kernel $K^{V'} \coloneqq \sum_{v \in V'} vv^T$. It is well-known that any DPP can be viewed as a mixture of its corresponding elementary DPPs~\cite{kulesza2012determinantal}, i.e.
\begin{align}
\Pr[K]{.} = \sum_{V' \subseteq V} \Big(\Pi_{v \in V'} \lambda_v \Pi_{v \notin V'} (1 - \lambda_v)\Big) \Pr[K^{V'}]{.}. \label{eq:elementaryDPPs}
\end{align}
Using this mixture formulation, we can sample a set from $\Pr[K]{.}$ as follows:
For each eigenvector $v \in V$, we sub-sample $v$ with probability $\lambda_v$ to obtain the random subset $V'$ of $V$, then we sample $\mathcal J_K$ from the elementary DPP with marginal kernel $K^{V'}$. We call this sampling scheme ``elementary sampling:''
\begin{itemize}
    \item $(1)$ For each $v \in V$, sample $y_v \sim \text{Bernoulli}(\lambda_v)$, add $v \in V'$ if $y_v = 1$.
    \item $(2)$ sample $\mathcal J_K \sim \Pr[K^{V'}]{.}$
\end{itemize}
According to the mixture formulation in Equation~\eqref{eq:elementaryDPPs}, the elementary sampling scheme samples $\mathcal J_K$ according to $\Pr[K]{.}$. 

One can readily see that the projected matrix $\Pi_z(K)$ has the same eigenvectors as $K$ but with corresponding eigenvalues $\{\bar{\lambda}_v\}_{v \in V}$, where
\begin{align} \label{eq:primedefinition}
\bar{\lambda}_v = \left\{\begin{matrix*}[l]
  \lambda_v  & \text{if} \,\, \lambda_v \in [z,1-z]\\ 
  z & \text{if} \,\, \lambda_v < z\\ 
1-z & \text{if} \,\, \lambda_v > 1-z
\end{matrix*}\right.
\end{align}
This fact follows from applying the \textit{$2$-Weilandt-Hoffman} inequality~\cite{tao2012topics} for the projection operator $\Pi_z(.)$. We can similarly sample $\mathcal J_{\Pi_z(K)} \sim \Pr[\Pi_z(K)]{.}$ with the above elementary sampling scheme. Next, we define a coupling between $\mathcal J_K$ and $\mathcal J_{\Pi_z(K)}$ as follows:
\begin{itemize}
    \item $(1)$ For each $v \in V$, sample $x_v \sim \text{Uniform}[0,1]$. Then add $v$ to $V'_1$ if $x_v \in [0,\lambda_v]$, and add $v$ to $V'_2$ if $x_v \in [0,\bar{\lambda}_v]$.
    \item $(2)$ if $V'_1 = V'_2$, then sample $\mathcal J \sim \Pr[K^{V'_1}]{.}$ and set $\mathcal J_K = \mathcal J_{\Pi_z(K)} = \mathcal J$. Otherwise, independently sample $\mathcal J_K \sim \Pr[K^{V'_1}]{.}$, $\, \mathcal J_{\Pi_z(K)} \sim \Pr[K^{V'_2}]{.}$.
\end{itemize}
By looking at the marginal distributions of the sets $\mathcal J_K$ and $\mathcal J_{\Pi_z(K)}$ sampled above, we observe that $\mathcal J_K \sim \Pr[K]{.}$, $\, \mathcal J_{\Pi_z(K)} \sim \Pr[\Pi_z(K)]{.}$, i.e. the marginals of the coupling are as one would expect. Furthermore, if the sampled sets $V'_1$ and $V'_2$ in the first step of the sampling are equal, then $\mathcal J_K = \mathcal J_{\Pi_z(K)}$. Therefore, to lower bound $\Pr[\text{coupling}]{\mathcal J_K = \mathcal J_{\Pi_z(K)}}$, it is enough to upper bound $\Pr[\text{coupling}]{\mathcal W}$ for the event $\mathcal W \coloneqq \{V'_1 \neq V'_2\}$. But we can expand $\mathcal W$ as 
$$\mathcal W = \bigcup_{v \in V} \Big(\{v \in V'_1, v \notin V'_2\} \cup \{v \in V'_2, v \notin V'_1\}\Big).$$
Note that for each $v \in V$, $\{v \in V'_1, v \notin V'_2\} \cup \{v \in V'_2, v \notin V'_1\}$ happens with probability $|\lambda_v - \bar{\lambda}_v|$. From Equation~\eqref{eq:primedefinition}, we observe that $|\lambda_v - \bar{\lambda}_v| \leq z$ for every $v \in V$. Therefore, using a union bound, we obtain 
$$\Pr[\text{coupling}]{\mathcal W} \leq nz.$$
Using the definition $z = \delta/2mn$, we conclude that  
\begin{align} \label{eq:couplingbound}
\Pr[\text{coupling}]{\mathcal J_K = \mathcal J_{\Pi_z(K)}} \geq 1 - \Pr[\text{coupling}]{\mathcal W} \geq 1 - nz = 1 - \frac{\delta}{2m}.
\end{align}

Using this coupling to generate the samples $\{\mathcal J^{(t)}_K\}_{t=1}^m$ and $\{\mathcal J^{(t)}_{\Pi_z(K)}\}_{t=1}^m$, we can write
\begin{align*}
\Pr[\text{coupling}]{ \{\mathcal J^{(t)}_K\}_{t=1}^m = \{\mathcal J^{(t)}_{\Pi_z(K)}\}_{t=1}^m} & = 
\Big( \Pr[\text{coupling}]{\mathcal J_K = \mathcal J_{\Pi_z(K)}} \Big)^m \\
& \geq \Big( 1 - \frac{\delta}{2m} \Big)^m 
\end{align*}
For a real number $u$, we have the inequality
$$(1-\frac{1}{u})^u \leq e^{-1},$$
and for $u \geq 2$, we have
$$(1-\frac{1}{u})^u \geq e^{-\frac{u}{u-1}} \geq e^{-2}.$$
Applying these inequalities, we finally obtain
\begin{align*}
\Pr[\text{coupling}]{ \{\mathcal J^{(t)}_K\}_{t=1}^m = \{\mathcal J^{(t)}_{\Pi_z(K)}\}_{t=1}^m} \geq \bigg(\Big( 1 - \frac{\delta}{2m} \Big)^{\frac{2m}{\delta}}\bigg)^{\frac{\delta}{2}} \geq  e^{-\delta}
\geq 1 - \delta.
\end{align*}
\end{proof}

%% file: appendixJ.tex
\section{A More Detailed Proof of Theorem~\ref{thm:upperboundtheorem}} \label{app:appendixj}
In this section, we take a more elaborate look at the proof of Theorem~\ref{thm:upperboundtheorem}. The proof is mentioned in Section~\ref{subsec:extention}.

\begin{proof}[Detailed proof of Theorem~\ref{thm:upperboundtheorem}]
Lemma~\ref{prop:refined} tells us there exists a constant $c_1$ such that $c_1 C_{N,\epsilon,\alpha,\zeta} \sqrt N/\epsilon^2$ samples suffice for \dpptester\ to successfully test against $(\alpha, \zeta)$-normal DPPs, with probability at least $0.995$. For the general problem of testing against any DPP (i.e. without having the normality conditions), we prove that $m^* = c_2 C_{N,\epsilon} \sqrt N/\epsilon^2$ samples suffice to succeed with probability at least $0.99$, as long as $c_2 \geq c_1 \max\{ 23, 2\log(c_1) + 23\}$. To test against all DPPs, we use the parameter setting of \dpptester\ for $(0,\bar z)$-normal DPPs, where we define $\bar z \coloneqq 0.005/2m^*n$. The key idea is that via the coupling argument of Lemma~\ref{lemma:coupling}, we can reduce the analysis for testing against all DPPs to the analysis for testing against only $(0,\bar z)$-normal DPPs. To this end, we use the following Lemma. The derivation of the inequality in Lemma~\ref{lem:sampleinequality} is based on elementary algebraic operations, and we differ its proof to the end of this section.
\begin{lemma}\label{lem:sampleinequality}
For constant $c_2$ picked as large as $c_2 \geq c_1 \max\{ 23, 2\log(c_1) + 23\}$, we have
\begin{align}
m^* \geq C_{N,\epsilon,0,\bar z} \sqrt N/\epsilon^2. \label{eq:constantsineq}
\end{align}
\end{lemma}

Therefore, we pick $c_2 \geq c_1 \max\{ 23, 2\log(c_1) + 23\}$ to satisfy the inequality $m^* \geq C_{N,\epsilon,0,\bar z} \sqrt N/\epsilon^2$. This means that given $m^*$ samples, according to the definition of $c_1$, our tester can test against $(0,\bar z)$-normal DPPs with success probability at least $0.995$. Therefore, if the underlying distribution $q$ is an $(0,\bar z)$-normal DPP, or if it is $\epsilon$-far from all DPPs, then \dpptester\ outputs correctly with probability at least $0.995$. It remains to show that the algorithm can also handle a DPP with kernel $K^*$, which is not $(0,\bar z)$-normal. To see this, note that because of the particular choice of $\bar z$, our coupling argument in Lemma~\ref{lemma:coupling} implies that the product distributions $\mathrm{\mathbf{Pr}}^{(m^*)}_{K^*}[.]$ and $\mathrm{\mathbf{Pr}}^{(m^*)}_{\Pi_{\bar z}(K^*)}[.]$ over the space of data sets have $\ell_1$-distance at most $0.005$. This follows from the fact that for two arbitrary random variables $X$ and $Y$ over the same underlying space, with probability distributions $P_X$ and $P_Y$, we have the following characterization of their $\ell_1$-distance:
$$\ell_1(P_X , P_Y) = \inf_{\text{coupling} (X,Y)} \Pr[\text{coupling}]{X \neq Y}.$$
Therefore, we have $\ell_1\Big(\mathrm{\mathbf{Pr}}^{(m^*)}_{K^*}[.], \mathrm{\mathbf{Pr}}^{(m^*)}_{\Pi_{\bar z}(K^*)}[.]\Big) \leq 0.005$. From this, we can relate the probability of the tester's acceptance region under $\mathrm{\mathbf{Pr}}^{(m^*)}_{K^*}[.]$, to the same probability under $\mathrm{\mathbf{Pr}}^{(m^*)}_{\Pi_{\bar z}(K^*)}[.]$:
$$\mathrm{\mathbf{Pr}}^{(m^*)}_{K^*}\left[ \text{Acceptance Region}\right]
\geq 
\mathrm{\mathbf{Pr}}^{(m^*)}_{\Pi_{\bar z}(K^*)}\left[ \text{Acceptance Region}\right] - 0.005 \geq 0.995 - 0.005 = 0.99,$$
where the last inequality follows from the fact that $\Pr[\Pi_{\bar z}(K^*)]{.}$ is an $(0,\bar z)$-normal DPP, according to the definition of $\Pi_{\bar z}(K^*)$. Hence, for $c_2 \geq \max\{ 23, 2\log(c_1) + 23\}$, \dpptester, with the particular choice of its parameter $\varsigma$ with respect to $(0,\bar z)$-normal DPPs, succeeds given $c_2C_{N,\epsilon}\sqrt N/\epsilon^2$ samples to test all DPPs with probability at least $0.99$. This completes the proof of Theorem~\ref{thm:upperboundtheorem}.
\end{proof}

\begin{proof}[Proof of Lemma~\ref{lem:sampleinequality}]
As usual, $\log(.)$ denotes the natural logarithm. Inequality~\eqref{eq:constantsineq} boils down to
$$c_2 C_{N,\epsilon} \geq c_1 C_{N,\epsilon,0,\bar z},$$
or equivalently
$$c_2 \log^2(N)(\log(N) + \log(1/\epsilon)) \geq c_1 \log^2(N)(1 + \log(1/\bar z) +  \log(1/\epsilon))$$
\begin{align}
\Leftrightarrow c_2 (\log(N) + \log(1/\epsilon)) \geq c_1(1 + \log(1/0.0025) + \log(m^*) + \log(n) +  \log(1/\epsilon)). \label{eq:secondequivalence}
\end{align}
Using the inequality $\log(x) \leq x - 1$ for $x > 0$, we get:
\begin{align*}
\log(m^*) & = \log(c_2 C_{N,\epsilon} \sqrt N/\epsilon^2) \\
& = \log(c_2) +2\log(\log(N)) + \log(\log(N) + \log(1/\epsilon)) + \frac{1}{2}\log(N) + 2\log(1/\epsilon) \\
& \leq \log(c_2) + 2(\log(N)-1) + \log(N) + \log(1/\epsilon) - 1 + \frac{1}{2}\log(N) + 2\log(1/\epsilon) \\
& = \log(c_2) -2 + \frac{7}{2}\log(N) + 3\log(1/\epsilon). \label{eq:thirdequivalence} \numberthis
\end{align*} 
Substituting Inequality~\eqref{eq:thirdequivalence} in Inequality~\eqref{eq:secondequivalence}, it is enough to satisfy
$$
\frac{c_2}{c_1} \geq \frac{\log(c_2) - 1 + \log(1/0.0025) + 7/2\log(N) + 4\log(1/\epsilon) + \log(n)}{\log(N) + \log(1/\epsilon)}  \coloneqq \varrho.
$$
We further upper bound $\varrho$ using the inequalities $\log(n) < \frac{1}{2}\log(N) + 1$ and $\log(N) \geq 0.69$:
\begin{align*}
\varrho & <  \frac{\log(c_2) + 6 + 8\log(N) + 4\log(1/\epsilon)}{\log(N) + \log(1/\epsilon)}
\\
& = \frac{\log(c_2) + 6}{\log(N) + \log(1/\epsilon)} + \frac{8\log(N) + 4\log(1/\epsilon)}{\log(N) + \log(1/\epsilon)}
\\
& \leq 1.5\log(c_2) + 9 + \frac{ 8(\log(N) + \log(1/\epsilon))}{\log(N) + \log(1/\epsilon)}
\\
& = 1.5 \log(c_2) + 17. 
\end{align*}
Therefore, it is enough to satisfy $c_2/c_1 \geq 1.5 \log(c_2) + 17$. But setting $c_2/c_1 = c_3$, this means we should choose $c_3$ large enough so that $c_3 \geq 1.5 \log(c_3) + 1.5\log(c_1) + 17$. One can readily check that $c_3 \geq \max\{ 23, 2\log(c_1) + 23\}$ satisfies this inequality. Consequently, it is enough to pick $c_2$ as large as $c_2 \geq c_1 \max\{ 23, 2\log(c_1) + 23\}$, which completes the proof of Lemma~\ref{lem:sampleinequality}. Note that $c_1 \max\{ 23, 2\log(c_1) + 23\}$ is almost a linear function of $c_1$.
\end{proof}

%% file: appendixD.tex
\section{Modification of \dpptester  \, for distinguishing $(\alpha, \zeta)$-normal DPPs from the $\epsilon$-far set of just the $(\alpha, \zeta)$-normal DPPs}
\label{app:appendixd}
Here, we explain how to manipulate the tester to work when we want to distinguish if $q$ is an $(\alpha, \zeta)$-normal DPP, or $\epsilon$-far only from the class of $(\alpha, \zeta)$-normal DPPs. 
We suggest that the reader first read the proof of Theorem~\ref{thm:learning}. 

The only part we change in the algorithm is the way we generate the set of candidate DPPs $\mathcal M$; we build the set of candidate marginal kernels $M$ the same way as in the proof of Theorem~\ref{thm:learning}. Given a candidate kernel matrix $K \in M$ and an arbitrary entry $K_{i,j}$, depending on whether $K_{i,j}$ is zero, or picked from the confidence interval around $\hat{K}_{i,j}^{(+)}$ or $\hat{K}_{i,j}^{(-)}$, we define the value $\alpha_{i,j}(K)$ to be zero, $+\alpha$, or $-\alpha$ respectively. Now when we are in the case where the underlying distribution is DPP, according to the way we generate $M$, with high probability there exists
a $\tilde K \in M$, such that $\tilde K_{i,j}$ is $\wp$-close to $K^*_{i,j}$ for every $i,j \in [n]$, and furthermore, $\alpha_{i,j}(\tilde K)$ is zero if $K^*_{i,j} = 0$, or has the same sign as $K^*_{i,j}$ if $K^*_{i,j} \neq 0$ ($\wp$ is defined in Equation~\eqref{sigmadefinition}). Our goal is to exploit this property of $\alpha_{i,j}(\tilde K)$'s to redefine $\mathcal M$, so that the candidate DPPs in $\mathcal M$ are $(\alpha, \zeta)$-normal. To this end, for each matrix $K \in M$,
instead of projecting $K$ onto the set of PSD matrices with eigenvalues in $[0,1]$, we project onto the following convex body with respect to the Frobenius distance, which is a subset of $(\alpha,\zeta)$-normal DPPs:
\begin{align*}
D_K \coloneqq \{A \in S_n^+| \, \, \zeta. I \preceq A \preceq & (1-\zeta) I, \, \forall i,j \in [n]: \\
\, & A_{i,j}/\alpha_{i,j}(K) \geq 1 \,\, \text{if} \, \alpha_{i,j}(K) \neq 0, \,\, \text{or} \, A_{i,j} = 0 \,\, \text{if} \, \alpha_{i,j}(K) = 0\},
\end{align*}
and generate $\mathcal M$ as 
$$\MM \coloneqq \{ \Pr[\Pi_{D_K}(K)]{.} | \, K \in M \},$$
where we denote by $\Pi_{D_K}$ the projection map onto $D_K$. Particularly, it is clear that $D_K$ is a subset of $(\alpha,\zeta)$-normal DPPs, and as the intersection of convex sets, $D_K$ is also convex, so projection on $D_K$ is well-defined. 

Now when $q$ is a DPP with marginal kernel $K^*$, we know it is $(\alpha,\zeta)$-normal, so for every $i,j \in [n]: \, |K^*_{i,j}| \geq \alpha$. Combining this with the property that $\alpha_{i,j}(\tilde K)$ is zero if $K^*_{i,j} = 0$, or it has the same sign as $K^*_{i,j}$ if $K^*_{i,j} \neq 0$, we obtain that $K^* \in D_{\tilde K}$. This means $\Pi_{D_{\tilde K}}(K^*) = K^*$. Using this relation with the contraction property of projection, we obtain 
$$\| \Pi_{D_{\tilde K}}(\tilde{K}) - K^* \|_F =
\| \Pi_{D_{\tilde K}}(\tilde{K}) - \Pi_{D_{\tilde K}}(K^*) \|_F \leq \| \tilde{K} - K^* \|_F.$$

Therefore, by substituting the projection $\Pi(K)$ in our algorithm by $\Pi_{D_K}(K)$ for every $K \in M$, the inequality in Equation~\eqref{eq:frob2} in the proof of Theorem~\ref{thm:learning} remains to hold, and the rest of the proof for the $\chi^2$-distance bound follows accordingly. On the other hand, with the new projection $\Pi_{D_K}(K)$ instead of $\Pi(K)$, the DPPs that are generated in $\mathcal M$ are all $(\alpha, \zeta)$-normal, so if we are in the case that $q$ is $\epsilon$-far from $(\alpha, \zeta)$-normal DPPs, it is also $\epsilon$-far from $\mathcal M$. Consequently, our $\chi^2$-$\ell_1$ tests are able to distinguish the two cases as before, and we obtain an $(\epsilon, 0.99)$-tester with sample complexity $\Theta(\sqrt N/\epsilon^2)$ for this modified version of our testing problem.

We should note that computing $\Pi_{D_K}(K)$ is trickier than $\Pi(K)$; for $\Pi(K)$, computing the Singular value decomposition (SVD) of $K$ is enough (or we can use iterative algorithms to get an approximate solution faster), but computing $\Pi_{D_K}(K)$ is a general convex problem and is solvable via convex programming approaches.

%% file: appendixC.tex
\section{Analysis of \dpptesterr}
\label{app:appendixc}

In this section, we show the argument in Theorem~\ref{thm:dpptester2}, which is a direct consequence of the sample and time complexities for the moment-based learning algorithm in~\cite{urschel2017learning}. 

\begin{proof}[Proof of Theorem~\ref{thm:dpptester2}]
Recall from the proof of Theorem~\ref{thm:learning} that estimating each entry of $K^*$ up to accuracy $\wp$, defined in Equation~\eqref{sigmadefinition}, is enough to prove the desired bound $\chi^2(q, \tilde p) \leq \epsilon^2/500$, which in turn enables the final $\chi^2$-$\ell_1$ tester to work correctly. 

Now let $\mathcal D_n$ be the set of $n \times n$ diagonal matrices with $+1$ or $-1$ on their diagonal. For any $D \in \mathcal D_n$, the marginal kernel $DK^*D$ induces the same DPP distribution as $K^*$ does. In other words, $K^*$ is identifiable only up to the multiplication of its rows and columns by $\pm 1$. With this in mind, to get the final guarantee for closeness of the DPP distributions when we use the moment-based learning algorithm, i.e. $\chi^2 \Big( q,\Pr[K^{\text{new}}]{.} \Big) \leq \epsilon^2/500$, it is enough that for some $D \in \mathcal D_n$, we estimate the matrix $DK^*D$ entrywise with accuracy $\wp$. In fact, the moment-based learning algorithm gives us such a guarantee; according to~\cite{urschel2017learning}, in order to compute a $\wp$-accurate estimate of $K^*$ in \textit{pseudo-distance}, the moment-based algorithm requires 
$O\bigg( \left(\frac{1}{\alpha^2\wp^2} + \ell (\frac{4}{\alpha})^{2\ell}\right)\log(n) \bigg)$ samples, where the pseudo-distance of matrices $K_1$ and $K_2$ is defined as
$$\rho(K_1, K_2) = \min_{D \in \mathcal D_n} \Big|DK_1D - K_2 \Big|_\infty = \min_{D \in \mathcal D_n} \max_{i,j \in [n]} \big|(DK_1D)_{i,j} - (K_2)_{i,j}\big|.$$
Now substituting $\wp$ from Equation~\eqref{sigmadefinition}, the sample complexity of the moment-based algorithm as a subroutine in~\dpptesterr \, becomes
 \begin{align}
 m = O\bigg(\frac{n^4 \log(n)}{\epsilon^2 \alpha^2  \zeta^2} + \ell (\frac{4}{\alpha})^{2\ell} \log(n) \bigg), \label{eq:momentsamplecomplexity}
 \end{align}
 where $\ell$ is the cycle sparsity\footnote{The cycle sparsity of a graph is the smallest $\ell'$ such that the cycles with length at most $\ell'$ constitute a basis for the cycle space of the graph.} of the graph with vertices $[n]$, whose edges correspond to the non-zero entries of $K^*$.
 
 Adding the complexity of the final $\chi^2$-$\ell_1$ test to the learning complexity in Equation~\eqref{eq:momentsamplecomplexity}, the overall sample complexity of \dpptesterr \, is:
 \begin{align*}
 O\bigg(\frac{n^4 \log(n)}{\epsilon^2 \alpha^2 \zeta^2} + \ell (\frac{4}{\alpha})^{2\ell} \log(n) + \frac{\sqrt N}{\epsilon^2}\bigg). \label{eq:alg2samplecomplexity}
 \end{align*}
For the time complexity, the run-time of the moment-based algorithm is $O(n^6 + mn^2)$ in the worst-case due to~\cite{urschel2017learning}, and the run-time of the $\chi^2$-$\ell_1$ test is $O(Nn^3 + m)$, as we have to compute $\Pr[K^{\text{new}}]{J}$ for each $J \subseteq [n]$, requiring an SVD in time $O(n^3)$. Adding them up results in an overall run time of
\begin{align*}
 O(Nn^3 + n^6 + mn^2) = O(\epsilon^4 m^2 n^3 + n^6 + mn^2) = \text{Poly}(m,n)
 \end{align*}
for~\dpptesterr, where the above equality follows from our sample complexity lower bound $m = \Omega(\sqrt N/ \epsilon^2)$.
\end{proof}

%% file: appendixF.tex
\section{Time complexity of \dpptester} \label{app:appendixf}
In this section, we analyze the time complexity of~\dpptester.

For each $p \in \mathcal M$, to apply the robust $\chi^2$-$\ell_1$ test of~\citet{ADK15}, one has to compute the statistic $Z^{(m)}$ defined in Equation~\eqref{eq:statistic}. To compute $Z^{(m)}$, one should compute $p(J)$ for every $J \subseteq [n]$, which requires a determinant calculation in time $O(n^3)$. 
Therefore, each robust $\chi^2-\ell_1$ testing takes time $O( N n^3)$. There is another $O(m)$ pre-processing time for computing $N(J)$'s. Moreover, computing the projection matrix $\Pi(K)$ for every $K \in M$ requires the Singular value decomposition (SVD) of $K$, which takes time $O(n^3)$. This is because we project with respect to the Frobenius distance, and it follows from the \textit{$2$-Weilandt-Hoffman} inequality~\cite{tao2012topics} that computing $\Pi(K)$ can equivalently be done by rounding down the eigenvalues of $K$ that are larger than one to one, and rounding up the eigenvalues that are negative to zero. Computing the initial estimate of the marginal kernel, i.e. $\hat K$ in the proof of Theorem~\ref{thm:learning}, also takes time at most $O(\min\{N,m\} n^2)$. Therefore, the overall time complexity becomes
$$O(|\mathcal M|N n^3 + m).$$

%% file: appendixE.tex
\section{Lower bound on the Sample Complexity of Distinguishing the Uniform distribution from $\mathcal F$}
\label{app:appendixe}
In this section, we give a high-level sketch of the approach that~\citet{DiakonikolasK:2016} use, to argue a lower bound of $\Omega(\sqrt N/\epsilon^2)$ on the sample complexity of the problem of testing the uniform distribution against $h_r$, randomly selected from $\mathcal F$.
\begin{proof}
Suppose that we observe samples from the underlying distribution $g$, where $g$ can either be $h_r$ or the uniform distribution. We flip a random coin $X$, and based on that set $g$ to the uniform distribution, or to $h_r$, a distribution randomly selected from $\mathcal F$. For every $S \subseteq [n]$, let $N(S)$ be the number of samples that are equal to $S$. We aim to show that given the number of samples satisfy $m = o(\sqrt N/\epsilon^2)$, the information in the collection of random variables $\mathcal A = \{N(S)| \, S \subseteq [n]\}$ is not enough to guess the value of $X$ strictly better than random guessing, say with success probability greater than $0.51$. 

To begin, we use the following Lemma without proof, which is exactly Lemma 3.2. in page 19 of~\cite{DiakonikolasK:2016}. This is a classical result in Information theory:
\begin{lemma} \label{lemma:first}
For random variables $X$ and $\mathcal A$, if there exist a function mapping $\mathcal A$ to $X$ such that $f(\mathcal A) = X$ with probability at least $0.51$, then we have the following bound on their mutual information:
$$I(X;\mathcal A) \geq 2.10^{-4}.$$ 
\end{lemma}
Based on Lemma~\ref{lemma:first}, it is enough to show that $I(X;\mathcal A) = o(1)$. To continue, we use the Poissonization trick; instead of directly deriving $m$ samples from $g$, we sample $m'$ from the Poisson distribution with parameter $m$, namely $m' \sim \text{Poisson(m)}$, then derive $m'$ samples from $g$. Using this trick, we still have $m' = \Theta(m)$ samples with high probability, so it is enough to bound $I(X,\mathcal A)$ for $\mathcal A$ with respect to the new sampling scheme with Poissonization. Based on properties of the Poisson distribution, the new scheme is equivalent to deriving $N(S) \sim \text{Poisson}(m g(S))$ for each set $S \subseteq [n]$ independent from the others. Furthermore, we showed in the proof of Theorem~\ref{thm:logsubmodular} that $L_r = \Theta(1)$ with high probability, so by using $m L_r$ instead of $m$ samples, the order of sample size does not change. But now, in the case $g = h_r$, $N(S)$ is sampled according to $N(S) \sim \text{Poisson}(m L_r h_r(S)) =\text{Poisson}(m \bar h_r(S))$. Thus, one can readily see that again, we can substitute $h_r$ by its unnormalized counterpart $\bar h_r$ in our Poisson sampling. 

Finally, assuming the sampling scheme $N(S) \sim \text{Poisson}(m \bar h_r(S)), \, \forall S \subseteq [n]$, we bound $I(X,\mathcal A)$. Note that given the value of $X$, the random variables $\{N(S)\}$ are independent, so we have the following bound on the mutual information:
\begin{equation}
I(X;\mathcal A) \leq \sum_{S \subseteq [n]} I(X;N(S)). \label{eq:importantt}
\end{equation} 
It is enough to bound each of the terms $I(X;N(S))$. For that, we bring without proof Lemma 3.3. from~\cite{DiakonikolasK:2016}, page 20:

\begin{lemma} \label{lemma:second}
If $N(S) \sim \text{Poisson}(m \bar h(S))$ for $X = 0$ and $N(S) \sim \text{Poisson}(m/N)$ for $X = 1$, then:
\begin{align*}
I(X;N(S)) = O(m^2\epsilon^4/N^2).
\end{align*}
\end{lemma}
From this Lemma and Equation~\eqref{eq:importantt}, we get
$I(X;\mathcal A) = o(m^2\epsilon^4/N)$ = o(1). Combining this with Lemma~\ref{lemma:first}, we conclude that we need at least $\Omega(\sqrt N/\epsilon^2)$ samples to non-trivially guess $X$ from the observed samples. This completes the proof of the promised lower bound on the sample complexity of the problem of testing uniform distribution against $\mathcal F$. For more details and the proof of Lemmas~\ref{lemma:first} and~\ref{lemma:second}, we refer the reader to~\cite{DiakonikolasK:2016}.
\end{proof}